\documentclass{article}

\usepackage{microtype}
\usepackage{graphicx}
\usepackage{subcaption}
\usepackage{booktabs} % for professional tables

\usepackage{hyperref}

\usepackage[accepted]{icml2026}

\usepackage{amsmath}
\usepackage{amssymb}
\usepackage{mathtools}
\usepackage{amsthm}
\usepackage{algorithm}
\usepackage{algorithmic}

\usepackage{array}
\usepackage[T1]{fontenc}
\usepackage[utf8]{inputenc}
\usepackage{multirow} 
\newcolumntype{C}{>{\centering\arraybackslash}m{1cm}}
\newcommand{\score}[3]{%
  \ensuremath{#1}%
  \ifx\relax#3\relax\else #3\fi%
  \std{#2}%
}
\newcommand{\na}{\textemdash}
\usepackage[capitalize,noabbrev]{cleveref}

\theoremstyle{plain}
\newtheorem{theorem}{Theorem}[section]

\newtheorem{lemma}[theorem]{Lemma}

\theoremstyle{definition}

\theoremstyle{remark}

\usepackage[textsize=tiny]{todonotes}

\icmltitlerunning{Controlled Dynamics Attractor Transformer}

\begin{document}

\twocolumn[
  \icmltitle{Controlled Dynamics Attractor Transformer}



  \icmlsetsymbol{equal}{*}

  \begin{icmlauthorlist}
    % \icmlauthor{Firstname1 Lastname1}{equal,yyy}
    % \icmlauthor{Firstname2 Lastname2}{equal,yyy,comp}
    % \icmlauthor{Firstname3 Lastname3}{comp}
    \icmlauthor{Cheng Zhang}{xjtu}
    \icmlauthor{Minnan Luo}{xjtu}
    \icmlauthor{Zesheng Yang}{xjtu}
    \icmlauthor{Ming Li}{thu}
    \icmlauthor{Yong-Jin Liu}{thu}
    \icmlauthor{Qinghua Zheng}{xjtu}
  \end{icmlauthorlist}

 \icmlaffiliation{xjtu}{School of Computer Science and Technology, MOEKLINNS Lab, Xi’an Jiaotong University,  Xi'an, Shaanxi, China}
  \icmlaffiliation{thu}{Department of Computer Science and Technology, Tsinghua University, Beijing, China}

\icmlcorrespondingauthor{Minnan Luo}{minnluo@xjtu.edu.cn}


  \icmlkeywords{Transformers,Energy Based Model,Hopfield Networks,Attractor Neural Network,Graph Anomaly Detection}
 \vskip 0.3in
]



\printAffiliationsAndNotice{}  % no special notice (required even if empty)

\begin{abstract}
Transformer architectures have dramatically advanced representation learning and inference in deep models through self-attention mechanisms. In parallel, associative memory frameworks map representations onto energy landscapes, offering interpretable retrieval mechanisms. However, their continuous-time inference dynamics lack the biological plausibility of classical Continuous Attractor Neural Networks. To bridge this gap, we propose \textbf{Controlled Dynamics Attractor Transformer (CDAT)}, which couples a mixture von Mises–Fisher attention energy with a Hopfield refinement energy, while augmenting energy descent with a CANN-inspired excitation–inhibition modulation. CDAT instantiates a topology-constrained dynamical system whose couplings encode relational structure among tokens, thereby linking attractor-style dynamics to modern energy-based attention. We further provide a constructive dissipation analysis to formally establish their controlled inference dynamics. Benefiting from these robust and structured dynamics, CDAT achieves state-of-the-art performance across multiple benchmarks in graph anomaly detection and graph classification.
\end{abstract}

\section{Introduction}
\label{submission}
Transformer-style attention has emerged as the cornerstone of information routing in modern deep networks. However, in realistic scenarios characterized by noisy observations, adversarial perturbations, and complex relational dependencies, the update dynamics of token representations often prove fragile. Small deviations can propagate and amplify through repeated interactions, causing trajectories to drift toward spurious states, and global mixing can wash out locally meaningful structure~\cite{pmlr-v139-davis21a}. This motivates a question that is largely orthogonal to block-level heuristics: can we explicitly endow attention-driven representation updates with a notion of stability and attraction, ensuring that iterative inference converges reliably to meaningful states rather than merely producing a feed-forward output?

A natural solution emerges from attractor dynamics in cognitive science and neural computation. Associative (semantic) memory explains how biological systems robustly recover latent content from partial cues~\cite{NIPS2016_eaae339c,krotov2021large}—a behavior computationally modeled as content-addressable memory. In this process, complete patterns are reconstructed from degraded queries by evolving neural activity toward stable states~\citep{hinton2014parallel, rolls2013mechanisms, tsodyks1995associative}. This perspective suggests modeling representation learning as an iterative state evolution that relaxes toward stable attractors shaped by a global objective.

Modern advancements in Associative Memory (AM) model make this connection concrete. Hopfield networks and their dense/continuous variants provide recurrent dynamics governed by an explicit energy function, and recent results reveal softmax attention as a special case of these models~\citep{ramsauer2021hopfield, NEURIPS2021_8171ac2c}. From this perspective, stability is intrinsic rather than incidental: the existence of a global energy landscape imposes rigorous constraints on permissible operations and parameter symmetries, theoretically guaranteeing the convergence of the forward computation~\citep{saha2023end}.

In addition, Energy Transformer (ET) operationalizes this principle by starting from a task-tailored energy function and deriving an iterative Transformer-like block as a consequence of minimizing that energy ~\citep{hoover2023energy}. Instead of stacking many conventional blocks, ET iterates token representations within a single energy-based block until (approximate) convergence, turning inference into an interpretable dynamical system that can be inspected through update directions over time.

%A key motivation for our design arises from the limitation of \textbf{discrete attractor} models: classic Hopfield-type networks excel at robust retrieval over finite pattern sets ~\citep{amit1985storing,amit1989model}, yet their basins typically collapse memories into isolated point attractors, which can be restrictive when latent factors vary continuously. Continuous Attractor Neural Networks (CANNs) offer a complementary perspective: under approximately symmetric interactions, they can sustain stable activity profiles that drift smoothly along low-dimensional manifolds ~\citep{amari1977dynamics, samsonovich1997path,NIPS2014_57c76ace}. In structured settings, however, relational interactions (\emph{e.g.}, graph topology) are usually heterogeneous~\citep{hamilton2017representation}, so the manifold view is only approximate. Practically, the dynamics still converge to discrete stable states, by absorbing relational information, producing basins that better align with the data geometry and thereby improving robustness and retrieval quality.

A key motivation for our design arises from the limitations of discrete attractor models. Classic Hopfield-type networks excel at robust retrieval over finite pattern sets \citep{amit1985storing,amit1989model}, yet their basins typically collapse memories into isolated point attractors, which becomes restrictive when latent factors vary continuously. Continuous Attractor Neural Networks (CANNs) offer a complementary perspective: under approximately symmetric interactions, they sustain stable activity profiles that drift smoothly along low-dimensional manifolds \citep{amari1977dynamics, samsonovich1997path, NIPS2014_57c76ace}. In structured settings, however, relational interactions (\emph{e.g.}, graph topology) are typically heterogeneous \citep{hamilton2017representation}, rendering the manifold view only approximate. Practically, the dynamics still converge to discrete stable states by absorbing relational information, producing basins that better align with the data geometry and thus improving robustness and retrieval quality.

Accordingly, we inject CANN-inspired mechanisms into modern Hopfield-type associative memory via a continuous-attractor–inspired modulation that implements local excitation and global inhibition under  structurally modulated feedback. The goal is not to enforce a literal continuous attractor manifold, but to reshape the energy landscape and its induced trajectories so that the resulting discrete attractors become topology-aware and less prone to spurious convergence. More broadly, this modulation introduces a new control interface over attractor dynamics: instead of treating inference as a purely monotone descent along the energy negative gradient~\citep{hoover2023energy}, we expose additional, principled degrees of freedom for steering, damping, and stabilizing the state evolution. This challenges the conventional view that associative memory dynamics need only be driven by energy descent~\citep{hopfield1982neural, ramsauer2021hopfield}, and positions our framework as a more expressive and controllable platform for iterative inference.

In this paper, we propose the Controlled Dynamics Attractor Transformer (CDAT), an energy-based Transformer framework with a clearer division of computational roles and a more controllable attractor geometry. CDAT integrates two complementary energy components—for coarse directional alignment and fine-grained prototype retrieval—with a CANN-inspired modulation mechanism~\citep{hamilton2017representation, wu2020comprehensive}. This design explicitly stabilizes the attractor geometry against spurious states while preserving structural fidelity. As illustrated in Fig.~\ref{fig:abstract}, this mechanism enables CDAT to escape spurious local minima that entrap standard energy descent schemes. 
Specifically, our main contributions are threefold:
\begin{itemize}
    \item  We propose CDAT, an energy-based Transformer framework that develops a structure-aware modulation mechanism inspired by CANNs. By coupling this with explicit stabilization, we achieve controllable attractor dynamics beyond naive energy descent. The resulting topology-induced attractors suppress spurious states and oscillations, improving robustness on standard graph benchmarks. Furthermore, we provide a dissipation analysis proving that the system dynamics monotonically minimize the energy functional and converge to a stable invariant set. 
    % Furthermore, we provide a \textbf{dissipation analysis} that the proposed dynamics decrease along the monotone functional and characterize the associated invariant set.
    \item  By operationalizing the equivalence between self-attention and the mixture of von Mises–Fisher (Mo–vMF) distribution, we formulate a Mo–vMF energy to drive rapid, coarse semantic alignment in the directional feature space. This is coupled with a Hopfield refinement energy for late-stage prototype retrieval and sharpening, establishing a clear division of computational roles along the relaxation trajectory.
    % By operationalizing the probabilistic equivalence between self-attention and the \textbf{mixture von Mises–Fisher (Mo–vMF)} probabilistic model, we formulate a Mo–vMF energy to drive rapid, coarse semantic alignment in directional feature space. This is coupled with a \textbf{Hopfield refinement energy} for late-stage prototype retrieval and sharpening, establishing a clear division of computational roles along the relaxation trajectory.
    \item We empirically demonstrate that CDAT achieves state-of-the-art performance across multiple benchmarks in graph anomaly detection and classification. It consistently outperforms strong energy-based baselines, validating that our controlled dynamics significantly enhance structural discriminability and detection stability against noise.

\end{itemize}

\begin{figure}[!t]
  \begin{center}
    \centerline{\includegraphics[width=\columnwidth]{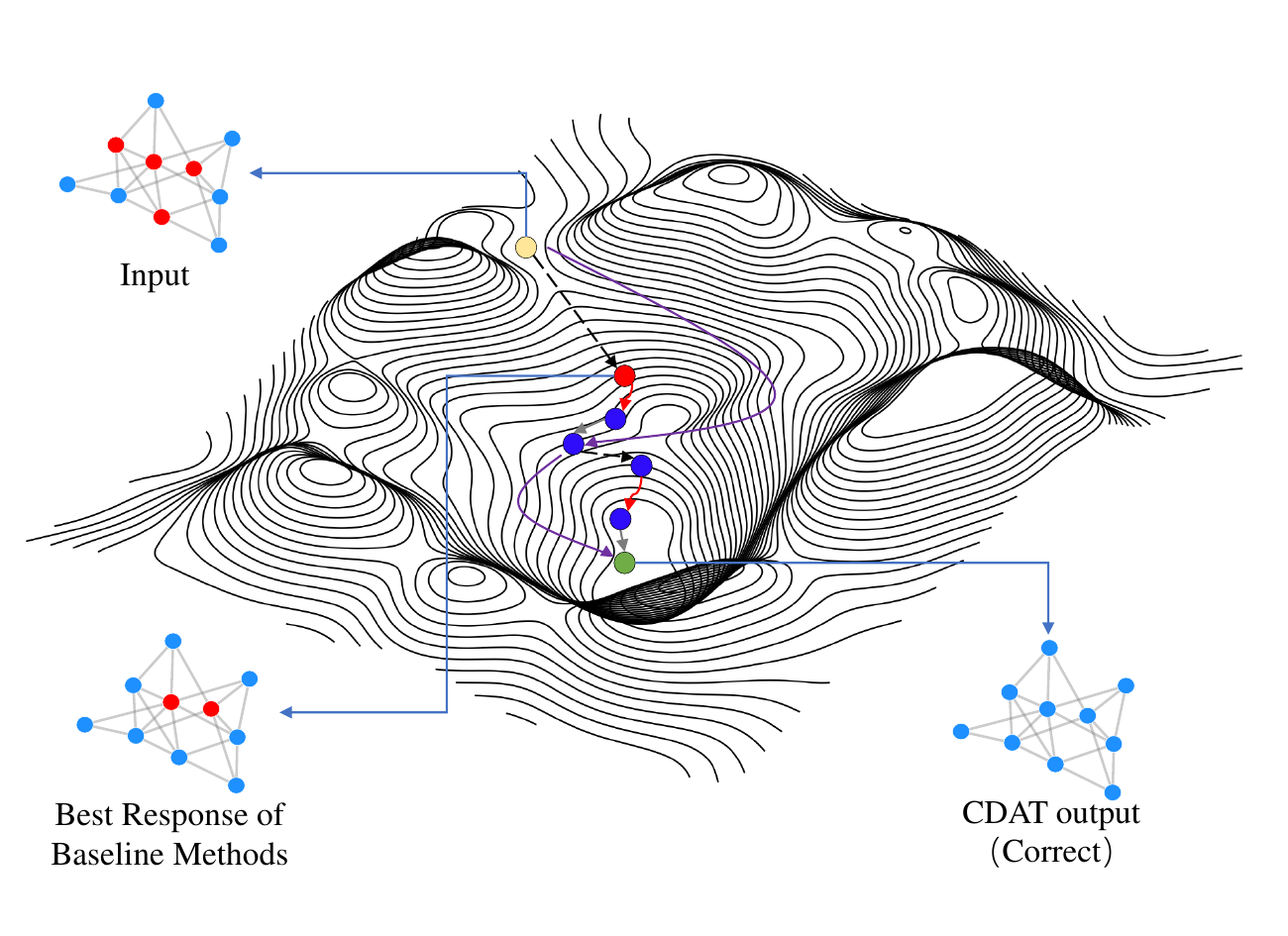}}
    \caption{CDAT dynamics.  The figure shows the representation trajectories on the energy landscape. While the negative-gradient baseline (black dashed) gets trapped in a spurious local minimum, CDAT (purple) reaches the global minimum via self-inhibition (red) that repels the state from shallow traps and topology-aware aggregation (gray) that steers it with structural priors.
    }
    \label{fig:abstract}
  \end{center}
\end{figure}

\section{Related Work}
\paragraph{Associative memory model.}
Dense associative memory networks can be viewed as a powerful generalization of the classical Hopfield network. They embed Hopfield-type dynamics into deep networks via energy-based associative memory mechanisms, thereby defining a continuous artificial recurrent neural network whose state is characterized by a vector evolving over time according to a nonlinear update rule~\citep{millidge2022universal,NEURIPS2024_29ff36c8,NEURIPS2023_57bc0a85}.
Within this framework, each stored pattern exerts a ``force'' on the state (or particle) in a way that follows a specific statistical model.
Beyond the classical capacity limitations of discrete Hopfield memories~\citep{hopfield1982neural}, dense variants substantially increase memory capacity and strengthen pattern completion by using rapidly growing nonlinearities / higher-order interactions~\citep{NIPS2016_eaae339c,demircigil2017model,lucibello2024exponential}, and theoretical analyses further characterize the large-capacity regime~\citep{krotov2021large,hu2024outlier}.
These developments motivate viewing representation learning through an attractor-memory lens; CDAT builds on this perspective but targets controllable and structure-aware attractor geometry for iterative inference.

\paragraph{Energy-based Transformer and iterative inference.}
Energy-based self-attention recasts the Transformer block as a differentiable dynamical system minimizing a global energy function~\citep{ramsauer2021hopfield}.
Rather than performing a single-step weighted sum, tokens evolve self-consistently in an energy landscape, reflecting internal information flow and the relaxation of the system towards equilibrium.
This view connects attention to modern Hopfield retrieval~\citep{ramsauer2021hopfield,NEURIPS2021_8171ac2c,Sun_2025_CVPR,wu2025incontext} and motivates architectures that perform inference by iterating within a block until (approximate) convergence~\citep{du2022learning,hoover2023energy}.
Relatedly, associative memory dynamics have been used as an end-to-end differentiable solver for prototype-style objectives such as clustering ~\citep{saha2023end}.
CDAT follows this iterative-inference paradigm but introduces a two-energy decomposition (Mo–vMF alignment vs.\ Hopfield refinement) and a topology-conditioned modulation to improve convergence under heterogeneous relational interactions.

\paragraph{Continuous attractors and excitation--inhibition mechanisms.}
Neural field models with local excitation and global inhibition can sustain continuous attractors~\citep{amari1977dynamics,wu2016continuous,chandra2025episodic}.
Additional mechanisms such as spike-frequency adaptation can enhance tracking behavior and modulate attractor dynamics ~\citep{NIPS2014_57c76ace}.
CDAT leverages excitation--inhibition as a structure-dependent modulation to reshape attractor basins and suppress spurious convergence, without requiring an exact continuous-attractor manifold.

\section{CDAT Dynamics of Token Updates}

In this section, we formalize CDAT as a continuous-time dynamical system. By starting with an energy-gradient drive, we augment the trajectory with CANN-inspired excitation–inhibition dynamics to enable topology-aware modulation. We establish a dissipation certificate for the proposed dynamics via a trajectory-wise storage functional, and discretize the resulting Ordinary Differential Equation (ODE) with a first-order Euler step to obtain the CDAT update rule. The specific semantic energy used is defined in Sec.~4 and the overall architecture of CDAT is shown in Fig.~\ref{fig:overview}.

\begin{figure}[t]
  %\vskip 0.2in
  \begin{center}
    \centerline{\includegraphics[width=\linewidth]{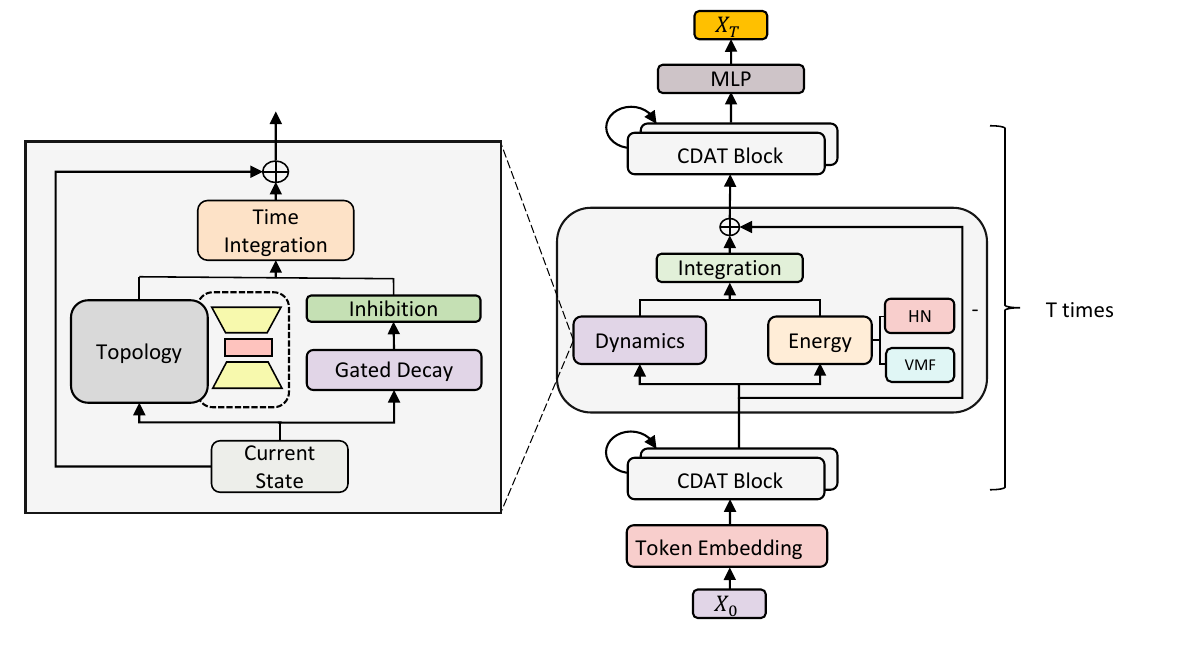}}
    \caption{Overall architecture of CDAT.}
  \label{fig:overview}
  \end{center}
\end{figure}

\subsection{Energy-guided Gradient Flow}
AM models achieve content-addressable retrieval by pairing an explicit energy function with attractor dynamics ~\citep{krotov2021large,millidge2022universal}: minima of the energy correspond to stable memories, and inference proceeds by driving the state toward attractors with low energy. In the same spirit of energy-based inference, we view token refinement as continuous-time \emph{state evolution} on an energy landscape. In a $D$-dimensional Euclidean space, consider a state $\mathbf{x}=(x_1,\cdots,x_D) \in \mathbb{R}^D$ and $\mathbf{g}=(g_1,\cdots,g_D) \in \mathbb{R}^D$ as its layer-normalized output. Particularly let $E: \mathbb{R}^{N \times D} \to \mathbb{R}$ be a differentiable energy function defined over this normalized state matrix $G=[\mathbf{g}_1,\cdots,\mathbf{g}_N]^\top \in \mathbb{R}^{N \times D}$. The attractor dynamics controls the trajectory of $\mathbf{x}$ in latent space by specifying $d\mathbf{x}/dt$; Under monotone energy decrease, that is $dE/dt < 0$, it guarantees stable convergence to a local minimum of the energy landscape.
We write this as the gradient-flow ODE:
\begin{equation}
\label{eq:grad_flow}
\tau \frac{d \mathbf{x}}{dt}
= -\nabla_{\mathbf{g}} E,
\end{equation}
where $\tau>0$ sets the time scale and lower energy corresponds to more compatible configurations.
This establishes a principled baseline dynamics on which we introduce the controlled excitation--inhibition modulation in the next subsection.

\begin{figure*}[!t]
  \begin{center}
    \centerline{\includegraphics[width=\linewidth]{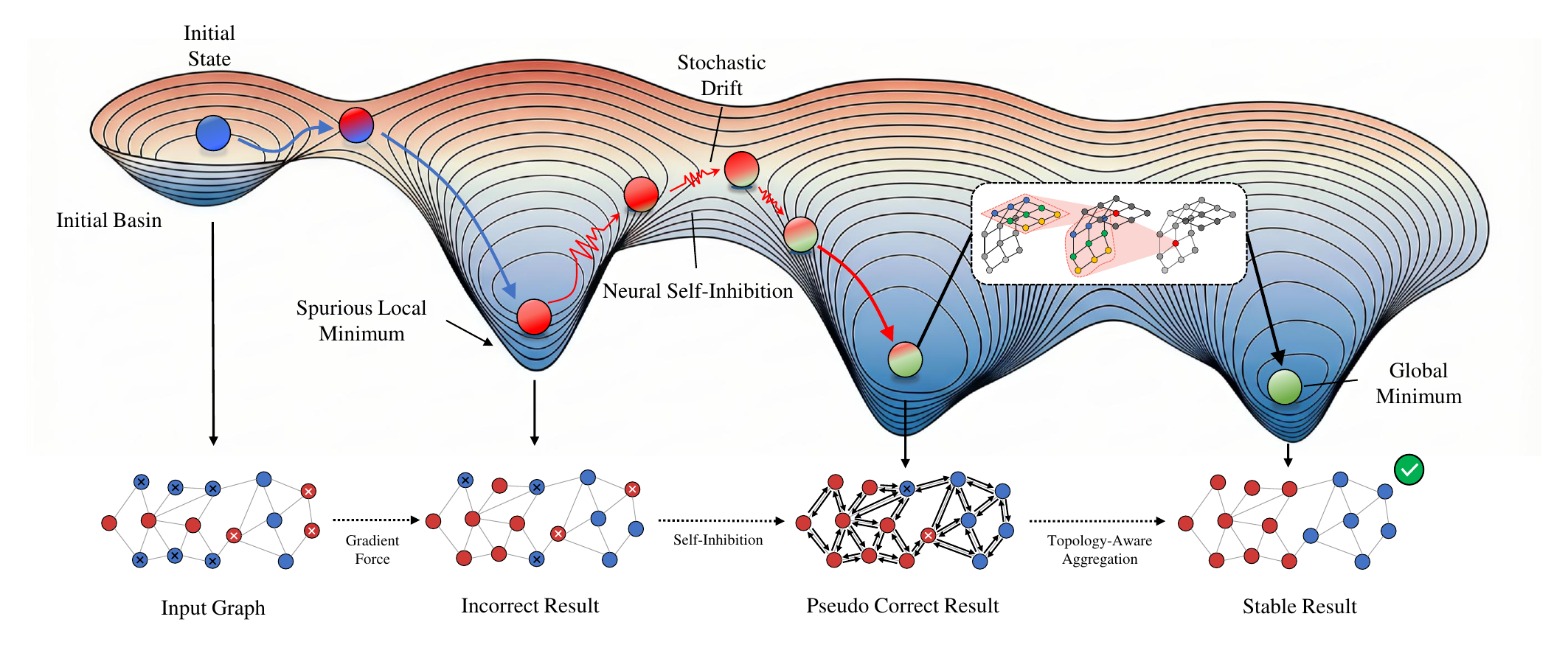}}
    \caption{Demonstration of Dynamics. Illustration of the classification trajectory, with the system state (the small sphere) evolving across the energy surface. The diagram depicts an initial descent into a Spurious Local Minimum (incorrect class), followed by an escape mechanism driven by Neural Self-Inhibition and Stochastic Drift (red wavy arrows). These terms propel the state through a pseudo-correct basin and, via Topology-Aware Updates, into the True Classification Basin (global minimum) for the correct stable result.}
  \label{fig:dynamic}
  \end{center}
\end{figure*}

\subsection{Discretized Dynamic Mechanisms}
While the gradient flow in \cref{eq:grad_flow} provides a fast route to a nearby local energy minimum, its trajectory is fully dictated by the local geometry of $E$.
As the data and interactions become more complex, the resulting energy landscape can become sharply curved and brittle, making the energy design increasingly delicate and prone to abrupt capacity collapse in associative memories ~\citep{sharma2022content}. Based on the observations above, we use a complementary strategy: instead of relying solely on shaping $E$, we seek a more controllable attractor dynamics that can steer states of all objects $X=[\mathbf{x}_1,\mathbf{x}_2,\cdots,\mathbf{x}_N]^\top \in \mathbb{R}^{N \times D} $ toward better basins through explicit excitation--inhibition regulation.
We therefore instantiate the drift field by augmenting the energy-gradient drive with decay, topology-conditioned coupling, and adaptive self-inhibition. For clarity, we present the dynamics for a single token state $\mathbf{x}(t)$ by
\begin{equation}
\tau \frac{d \mathbf{x}}{dt}
= -\mathbf{x}
+  W\mathbf{x}
- \omega \mathbf{x}
- \nabla_{\mathbf{g}} E.
\label{eq:graph_cann_energy}
\end{equation}
\cref{eq:graph_cann_energy} decomposes the drift into four terms: a passive decay $-\mathbf{x}$, a topology-dependent excitation term parameterized by $W \in \mathbb{R}^{D \times D}$, a global damping term $-\omega \mathbf{x}$ where $\omega \in \mathbb{R}_{>0}$ is a learnable scalar decay parameter and the task-directed energy drive $-\nabla_{\mathbf{g}} E$.
Together, these terms define a controllable attractor dynamics that improves basin search while retaining stable convergence.

\paragraph{Topology-conditioned mutual excitation $W\mathbf{x}$.}
For structured token interactions, an analogous construction can be achieved by defining a
distance-based coupling kernel over an interaction structure.
Concretely, let feature dimensions be indexed by $\mathcal{D}=\{1,\cdots,D\}$ and endow $\mathcal{D}$ with a task-dependent
distance $d:\mathcal{D}\times\mathcal{D}\rightarrow \mathbb{R}_{\ge 0}$. Thus, we parameterize the coupling matrix $W_{ij}$ by decomposing it into a fixed geometric prior and a learnable residual, \emph{i.e.}
\begin{equation}
\label{eq:Wij_kernel}
W_{ij}
=
J_0 \exp\!\Bigl(-\frac{d(i,j)^2}{2a^2}\Bigr)
+
W_{ij}^{\mathrm{train}},
\end{equation}
where the Gaussian term imposes a locality bias, with $J_0$ scaling the excitation strength and $a$ controlling the receptive range; $W_{ij}^{\mathrm{train}}$ provides a global correction for long-range dependencies.
Crucially, the coupling is heterogeneous in feature space: the neighborhood structure induced by $d(i,j)$ is generally non-uniform across dimensions, which breaks the translation-like symmetry required for continuous attractors.
As a result, activity bumps are no longer equivalent under shifts on the feature manifold; instead they become biased toward structurally preferred regions (\emph{i.e.}, dimensions with denser or stronger effective couplings), yielding a finite set of stable fixed points rather than a continuum of equivalent states~\citep{seeholzer2019stability}.

\paragraph{Global damping $-\omega \mathbf{x}$ for energy dissipation and stabilization.}
The self-inhibition term $-\omega \mathbf{x}$ can be viewed as a coarse linear model of spike-frequency adaptation (SFA), implementing a delayed negative-feedback loop whereby sustained activity elicits an opposing inhibitory current that counteracts excitation~\citep{gutkin2014spike}.
From a dynamical systems perspective, such negative feedback can destabilize spatially localized steady-state attractors. As shown in Fig.~\ref{fig:dynamic}, in the context of energy landscape optimization, the resulting intrinsic velocity plays a role analogous to inertial momentum: when the trajectory becomes trapped in shallow local minima, accumulated activity generates a repulsive effect that pushes the state away from its current location~\citep{NIPS2014_57c76ace}. Consequently, self-inhibition facilitates barrier crossing and promotes convergence toward deeper minima that correspond to the true class centers.

\subsection{Euler Discretization and Layerwise Interpretation}

As discussed above, CDAT goes beyond plain negative-gradient descent by introducing a controlled, higher-order dynamical system that governs trajectory evolution, rather than merely shaping the energy landscape. While the design is grounded in classical excitation--inhibition attractor mechanisms, our goal is not to reproduce a specific CANN model~\citep{amari1977dynamics}, but to expose a control interface that is both principled and implementation-friendly.
% In particular, mutual excitation and self-inhibition induce two complementary components after discretization:
% (i) a parameter-efficient, topology-conditioned convolution/message-passing operator, and
% (ii) an adaptive inhibition (gain-control) term that acts as a learnable state-dependent gating, yielding Langevin-like mobility in the trajectory.

In this view, the resulting discretized updates map directly to modern deep architectures. The excitation term corresponds to an additional structural mixing layer, while the inhibition term behaves like an extra pointwise transformation on the updated state. To make this connection explicit, we define the step size $\alpha:=\Delta t/\tau$ and apply a first-order (Euler) discretization to the dynamics in \cref{eq:graph_cann_energy}, yielding the state evolution of $\mathbf{x}$ at step $t$ as
\begin{equation}
\begin{split}
{\mathbf{x}}^{t+\Delta t}
&=
{\mathbf{x}}^t
+
\alpha \cdot
\Bigl[
W \cdot {\mathbf{x}}^t
-
(\omega+1) {\mathbf{x}}^t-\nabla_{ \mathbf{g}} E^t
\Bigr].
\end{split}
\label{eq:update_rule_en}
\end{equation}
We next establish a theoretical link between the proposed neurodynamics and modern neural networks. In particular, we show that the Euler-discretized evolution is isomorphic to a residual convolutional layer~\citep{lecun2002gradient} with a strong local-connectivity inductive bias. This viewpoint unifies message passing on graphs and convolution on grids as two instances of the same operator form.

\begin{theorem}[Excitation increment as residual convolution]
\label{prop:resgcn_equivalence}
Consider the Euler-discretized update of the proposed dynamics~\eqref{eq:update_rule_en}.
When the coupling matrix $W$ in the mutual-excitation term is instantiated using the kernel-based parametrization in~\eqref{eq:Wij_kernel}, the resulting mutual-excitation increment $\alpha W \mathbf{x}^{t}$ implements a feature-space convolution, corresponding to a globally parameterized aggregation operator defined by $W$. This structure naturally induces a ResCNN-style residual update, i.e.,
\begin{align}
  \mathbf{x}^{t+\Delta t}=\underbrace{\mathbf{x}^{t}}_{\text{Residual}}+\underbrace{\alpha W\mathbf{x}^{t}}_{\text{Convolution}}.  
\end{align}
\end{theorem}
Theorem~\ref{prop:resgcn_equivalence} establishes an explicit correspondence between the Euler-discretized dynamics and a feature-space convolutional update: the mutual-excitation term instantiates aggregation through a Gaussian heat-kernel operator, further modulated by learnable coupling weights. See Appendix~\ref{app:gcn-proof} for the proof.

\subsection{Dissipation Analysis}
To certify the dissipative structure of the modified dynamics (\cref{eq:graph_cann_energy}), 
we introduce the following storage functional along trajectory as
\begin{equation}
\label{eq:Vt_def}
V(t)=E(\mathbf{g}(t))
+\int_{t_0}^{t}\Bigl(\bigl((1+\omega) \mathbf{x}(s)-W \mathbf{x}(s)\bigr)^\top \dot{ \mathbf{g}}(s)\Bigr)\,ds,
\end{equation}
where $\mathbf{g}(t)=\mathbf{g}(\mathbf{x}(t))$ for brevity; $t_0$ is an arbitrary initial time.
The integral is well-defined since $\mathbf{x}(\cdot)$ and $\mathbf{g}$ are continuously differentiable functions of one order, hence $\dot{ \mathbf{g}}(\cdot)$ is continuous.
We next establish the key Jacobian property in Lemma~\ref{lem:LN_Jacobian_PSD} ~\citep{tang2021remarkpaperkrotovhopfield,DBLP:journals/corr/abs-2107-06446} before proving dissipation.

\begin{lemma}[LayerNorm Jacobian is PSD]
\label{lem:LN_Jacobian_PSD}
%Let $ \mathbf{x}\in\mathbb R^D, \mathbf{1} \in\mathbb R^D$ be the all-ones vector and define the mean $\mu(\mathbf{x})$ and standard deviation $\sigma( \mathbf{x})=\sqrt{\frac{1}{D}\| \mathbf{x}-\mu( \mathbf{x}) \cdot \mathbf{1}\|_2^2+\varepsilon}$ for $\varepsilon>0$.
Let $\mu(\mathbf{x})$ and $\sigma( \mathbf{x})$ be the mean and standard deviation over $\mathbf{x}\in\mathbb R^D$, with%这儿说的详细些，\mathbf{x}是什么？
\begin{align*}
    \sigma(\mathbf{x})=\sqrt{\frac{1}{D}\|\mathbf{x}-\mu(\mathbf{x})\cdot\mathbf{1}\|_2^2+\varepsilon},\ \varepsilon>0,
\end{align*}
where $\mathbf{1}\in\mathbb R^D$ is the all-ones vector.
For $\mathbf{g}(\mathbf{x})=\gamma\,\frac{ \mathbf{x}-\mu( \mathbf{x}) \cdot \mathbf{1}}{\sigma( \mathbf{x})}$ with $\gamma>0$, its Jacobian matrix $M(\mathbf{x})=\nabla_{ \mathbf{x}} \mathbf{g}(\mathbf{x})$ is symmetric positive semidefinite.
\end{lemma}

\begin{theorem}[Dissipation certificate for the modified dynamics]
\label{thm:dissipation_certificate}
Let $\mathbf{x}(t)$ be any solution of dynamics
\begin{align*}
\tau \dot{ \mathbf{x}} &= - \mathbf{x} + W \mathbf{x} - \omega  \mathbf{x} - \nabla_{ \mathbf{g}}E,
\end{align*}
with $\ W^\top=W$. Suppose $\mathbf{g}$ represents the Layer-normalized output of $\mathbf{x}$ as defined in Lemma ~\ref{lem:LN_Jacobian_PSD}. Then the trajectory-wise storage functional $V(t)$ for any fixed $t_0$ as ~\cref{eq:Vt_def} is non-increasing, i.e.,
\begin{align}
   \frac{d}{dt}V(t)=-\tau\,\dot{ \mathbf{x}}(t)^\top M( \mathbf{x}(t))\,\dot{ \mathbf{x}}(t) \le 0,
\end{align}
where $M(\mathbf{x}(t))=\nabla_{\mathbf{x}}\mathbf{g}(\mathbf{x}(t))$ is the Jacobian matrix of $\mathbf{g}$.% respect to $\mathbf{x}$.%; $\mathbf{g}(\mathbf{x}(t))=\gamma\,\frac{ \mathbf{x}(t)-\mu( \mathbf{x}(t)) \cdot \mathbf{1}}{\sigma( \mathbf{x}(t))}$ with $\gamma>0$.% and $W=W^\top$. 
%By the chain rule, $\dot{ \mathbf{g}}(t)=M( \mathbf{x}(t))\dot{ \mathbf{x}}(t)$. Then, if $M( \mathbf{x}(t))\succeq 0$, we have 
%Hence $V(t)$ is non-increasing.
\end{theorem}

%\begin{theorem}[Dissipation certificate for the modified dynamics]
%\label{thm:dissipation_certificate}
%Assume $\mathbf{g}:\mathbb{R}^D\to\mathbb{R}^D$ and $E:\mathbb{R}^{N \times D}\to\mathbb{R}$ are continuously differentiable functions of one order, the coupling matrix is symmetric, i.e., $W=W^\top$. For fixed parameters (including the learnable scalar $\omega$ during the rollout), let $\mathbf{x}(t)$ be any solution of
%\[
%\begin{aligned}
%\tau \dot{ \mathbf{x}} &= - \mathbf{x} + W \mathbf{x} - \omega  \mathbf{x} - \nabla_{ \mathbf{g}}E.
%\end{aligned}
%\]
%Define the trajectory-wise storage functional $V(t)$ (for any fixed $t_0$) as ~\eqref{eq:Vt_def}.
%Let $M( \mathbf{x}):=\nabla_{ \mathbf{x}} \mathbf{g}( \mathbf{x})$. By the chain rule, $\dot{ \mathbf{g}}(t)=M( \mathbf{x}(t))\dot{ \mathbf{x}}(t)$. Then, if $M( \mathbf{x}(t))\succeq 0$, we have 
%\[
%\frac{d}{dt}V(t)
%=
%-\tau\,\dot{ \mathbf{x}}(t)^\top M( \mathbf{x}(t))\,\dot{ \mathbf{x}}(t) \le 0.
%\]
%Hence $V(t)$ is non-increasing.
%\end{theorem}

The monotonic decrease of $V(t)$ establishes that the proposed dynamics is energy-dissipative. As in prior work on Associative Memory models~\citep{saha2023end}, we use this property to characterize the attractor structure of the dynamics. A detailed proof is deferred to Appendix~\ref{app:dissipation-proof}.

\paragraph{Symmetric, low-rank parameterization of the coupling.}
The dissipation result in Theorem~\ref{thm:dissipation_certificate} assumes a symmetric coupling $W=W^\top$.
This is both theoretically and practically motivated.
Theoretically, symmetry aligns the mutual-excitation term with an energy-based interaction and avoids non-conservative rotational components.
On the implementation side, we enforce symmetry by parameterizing
\begin{equation}
W \;=\; P^\top Q P,
\label{eq:sym_lowrank_W}
\end{equation}
where $P\in\mathbb{R}^{r\times D}$ with $r\ll D$; $Q\in\mathbb{R}^{r\times r}$ is diagonal (or block-diagonal).
By construction, $W$ is symmetric, and its rank is at most $r$ when $Q$ is full-rank.
This yields a parameter-efficient update rule reminiscent of LoRA:
the mutual-excitation operator is obtained from a low-dimensional bottleneck $P$ and a lightweight scaling $Q$,
reducing the number of free parameters from $\mathcal{O}(D^2)$ to $\mathcal{O}(rD)$ while retaining expressive, learnable couplings~\citep{hu2022lora}.
In our experiments, this constraint stabilizes training and improves the robustness of the controlled dynamics.

\section{Associative Memories for \textbf{CDAT}}%: An Energy-Based Framework with Von Mises–Fisher Probabilistic Modeling}
The controlled dynamics established in Section~3 (\cref{eq:graph_cann_energy}) are driven by the gradient of a semantic potential, $-\nabla_{ g}E$. In this section, we instantiate this energy $E$ to explicitly guide tokens toward consistent configurations via two complementary forces: input-conditioned alignment and prototype-based refinement.
Given states $X\in\mathbb{R}^{N\times D}$ and relational structure $\mathcal{G}$, we define the semantic energy of CDAT as
\begin{equation}
\label{eq:CDAT_total_energy}
E(X;\mathcal{G})
=
\lambda_{\mathrm{v}} E^{\mathrm{ATT-vMF}}(X)
+
\lambda_{\mathrm{h}} E^{\mathrm{HN}}(X),
\end{equation}
where $E^{\mathrm{ATT-vMF}}$ induces directional semantic alignment via a Mo–vMF interpretation of self-attention, and $E^{\mathrm{HN}}$ implements Hopfield-type prototype retrieval and sharpening.
$\lambda_{\mathrm{v}}$ and $\lambda_{\mathrm{h}}$ represent the balancing weights, satisfying $\lambda_{\mathrm{v}} + \lambda_{\mathrm{h}} = 1$.

% \begin{figure}[t]
%   \vskip 0.2in
%   \begin{center}
%     \centerline{\includegraphics[width=0.8\columnwidth]{figures/DD_vmf.pdf}}
%     \caption{
%       \textbf{Mo--vMF view of self-attention on DD.}
% Visualizing learned geometry on the DD dataset reveals hyperspherical clustering, where the induced directions $\mu_B$ act as dynamic semantic prototypes. Crucially, these prototypes are input-dependent, contrasting with the static stored memories of classical Hopfield networks.
%     }
%     \label{fig:movmf}
%   \end{center}
% \end{figure}

\subsection{Mo-vMF Term $E^{\mathrm{ATT-vMF}}$: Self-attention as Directional Mixture Modeling}
\label{sec:CDAT-movmf}

Energy-based self-attention~\citep{ramsauer2021hopfield} casts a Transformer attention block as minimizing an explicit energy function. This interpretation motivates us to treat token interactions as a differentiable dynamical system parameterized by the usual query/key projections. %In the following subsections, we mainly demonstrate that under unit-norm queries/keys, the attention energy function is equivalent to the Mo-vMF mixture likelihood form, thereby interpreting attention as a special case of directional mixture models.
Consider a state vector $\mathbf{x}$ and its LayerNorm-normalized output $\mathbf{g}$. For each head $h \in \{1,\cdots,H\}$, we obtain query and key vectors by linearly embedding
$\mathbf{g}$ into an internal feature space of dimension $Y$ via learnable tensors
$\mathbf{W}^Q,\mathbf{W}^K \in \mathbb{R}^{Y\times H\times D}$. The query and key vectors are obtained as 
\begin{align}
    Q_{hi} = \frac{\mathbf{W}^Q_h \mathbf{g}_i}{\|\mathbf{W}^Q_h \mathbf{g}_i\|_2},\ 
    K_{hi} = \frac{\mathbf{W}^K_h \mathbf{g}_i}{\|\mathbf{W}^K_h \mathbf{g}_i\|_2}.
\end{align}
%$Q_{hi} = \frac{\mathbf{W}^Q_h \mathbf{g}_i}{\|\mathbf{W}^Q_h \mathbf{g}_i\|_2} $ and $K_{hi} = \frac{\mathbf{W}^K_h \mathbf{g}_i}{\|\mathbf{W}^K_h \mathbf{g}_i\|_2}$.
The energy of the attention mechanism is defined as% ($B,C \in \{1,\cdots,N\}$)
\begin{equation}
\label{eq:attn-movmf}
\begin{aligned}
E^{\text{ATT-vMF}}
&= -\frac{1}{\beta} 
   \sum_{h=1}^{H} \sum_{C=1}^{N} 
   \log \Biggl(
      \sum_{B \ne C} 
      \exp\bigl( \beta \, Q_{hB}^\top K_{hC} \bigr)
   \Biggr).
\end{aligned}
\end{equation}

We connect the log-sum-exp term in $E^{\text{ATT-vMF}}$ to a probabilistic mixture model on the unit sphere
$S^{Y-1}=\{l\in\mathbb{R}^{Y}:\|l\|=1\}$. The von Mises--Fisher distribution on $S^{Y-1}$ with mean direction
$\mu\in S^{Y-1}$ and concentration $\beta$ has density
\begin{equation}
\label{eq:vmf_def}
p_{\mathrm{vMF}}(q\mid \mu,\beta)
=
c_Y(\beta)\exp\!\bigl(\beta\,\mu^\top q\bigr),
\end{equation}
where $q\in S^{Y-1}$ and $c_Y(\beta)$ is a normalizing constant depending only on $(Y,\beta)$.
Thus, Mo--vMF with shared concentration $\beta$ is
\begin{equation}
\label{eq:movmf_def}
p_{\mathrm{Mo\text{-}vMF}}(q)
=
\sum_{b=1}^{M}\pi_b\, p_{\mathrm{vMF}}(q\mid \mu_b,\beta),
\end{equation}
where $\pi_b\ge 0$, $\sum_{b=1}^{M}\pi_b=1$, and $\mu_b\in S^{Y-1}$ for all $b$.

\begin{theorem}[Self-attention energy as a Mo–vMF negative log-likelihood]
\label{prop:movmf}
Consider single-head self-attention situation with normalized query and
key vectors $Q_B,K_C\in S^{Y-1}$, we define the attention energy with $\beta>0$ as
\begin{equation}
\label{eq:att_energy}
E^{\mathrm{ATT-vMF}}
=
-\frac{1}{\beta}
\sum_{C=1}^{N}
\log\!\left(
\sum_{B\neq C}
\exp\bigl(\beta\, Q_B^\top K_C\bigr)
\right).
\end{equation}
Then, ignoring constant terms independent of $\{Q_B,K_C\}$, $E^{\mathrm{ATT-vMF}}$ is equivalent to the negative log-marginal likelihood of a Mo-vMF model on the hypersphere $S^{Y-1}$ with shared concentration $\beta$. 
\end{theorem}

\cref{prop:movmf} connects attention to directional mixture modeling, where query vectors act as mixture directions and updates encourage angular alignment. A formal statement and proof are deferred to Appendix~\ref{app:movmf-proof}.

We emphasize that the Mo-vMF interpretation serves as a geometric prior defining the attractor basins. While consistent with recent insights~\citep{schaeffer2023associative}, our contribution lies in operationalizing this equivalence to shift the modeling perspective: from implicit energy descent to explicit probabilistic evolution. This theoretical grounding justifies the decomposition in \cref{eq:CDAT_total_energy}, ensuring that the coarse alignment is driven by a rigorous spherical mixture prior.

\begin{table*}[t]
  \caption{%Graph classification performance on seven datasets of TUDataset (higher is better). Following~\citep{morris2020tudataset}, mean and standard deviation obtained from 100 runs of 10-fold cross validation are reported. Abbreviations: MUTAGEN.=MUTAGENICITY.
  Graph classification performance on seven benchmarks (higher is better). Following \citep{morris2020tudataset}, means and standard deviations over 100 runs of 10‑fold cross‑validation are reported. Abbreviation: MUTAGEN. = MUTAGENICITY.}
  \label{tab:classification}
  \begin{center}
%    \begin{small}
%      \begin{sc}
        \begin{tabular}{lccccccc}
\toprule
\textbf{Method} &
\textbf{PROTEINS} &
\textbf{NCI1} &
\textbf{NCI109} &
\textbf{DD} &
\textbf{ENZYMES} &
\textbf{MUTAG} &
\textbf{MUTAGEN.} \\
\midrule

WKPI (k-means)
& $78.5{\pm}0.4$
& $87.5{\pm}0.5$
& $85.9{\pm}0.4$
& $82.0{\pm}0.3$
& \na
& $85.8{\pm}2.5$
& \na \\

GRDL
& $82.6{\pm}1.2$
& $80.4{\pm}0.8$
& \na
& \na
& \na
& $92.1{\pm}5.9$
& \na \\

\midrule

DSGCN
& $77.3{\pm}0.4$
& \na
& \na
& \na
& $78.4{\pm}0.6$
& \na
& \na \\

Norm-GN
& \na
& $84.9{\pm}1.7$
& $83.5{\pm}1.3$
& \na
& $73.3{\pm}8.0$
& \na
& \na \\

GIN + GraphNorm
& $77.4{\pm}4.9$
& $81.4{\pm}2.4$
& $82.4{\pm}1.7$
& \na
& \na
& $91.6{\pm}6.5$
& \na \\

GIN + GRANOLA
& $77.5{\pm}3.7$
& $84.0{\pm}1.7$
& $83.7{\pm}1.6$
& \na
& \na
& $92.2{\pm}4.6$
& \na \\

\midrule

HGP-SL
& $84.9{\pm}1.6$
& $78.5{\pm}0.8$
& $80.7{\pm}1.2$
& $81.0{\pm}1.3$
& $68.8{\pm}2.1$
& \na
& $82.2{\pm}0.6$ \\

HGP
& $79.4{\pm}3.1$
& $74.2{\pm}1.7$
& \na
& $72.8{\pm}5.4$
& $44.5{\pm}7.4$
& $87.9{\pm}5.7$
& $77.9{\pm}1.4$ \\

ASAP
& $74.2{\pm}0.8$
& $71.5{\pm}0.4$
& $70.1{\pm}0.6$
& $76.9{\pm}0.7$
& \na
& \na
& \na \\

MinCutPool
& $74.7{\pm}0.5$
& $74.3{\pm}0.9$
& \na
& \na
& \na
& $90.6{\pm}4.6$
& \na \\

\midrule

U2GNN
& $80.0{\pm}3.2$
& \na
& \na
& \na
& $95.7{\pm}1.9$
& \na
& $88.5{\pm}7.1$ \\

Graphormer
& $76.3{\pm}2.7$
& $78.6{\pm}2.1$
& \na
& \na
& \na
& $89.6{\pm}6.2$
& \na \\

ET
& \textbf{90.3${\pm}$0.7}
& \underline{$90.1{\pm}0.1$}
& \underline{$90.5{\pm}0.1$}
& \underline{$95.9{\pm}0.8$}
& \textbf{99.8}
& \underline{$96.6{\pm}0.2$}
& \underline{$98.7{\pm}0.1$} \\

\midrule
\textbf{CDAT (ours)}
& \underline{$87.1{\pm}0.4$}
& \textbf{93.5${\pm}$0.2}
& \textbf{94.3${\pm}$0.1}
& \textbf{98.5${\pm}$0.6}
& \underline{$99.1{\pm}0.1$}
& \textbf{98.8${\pm}$0.3}
& \textbf{98.8${\pm}$0.1} \\

\bottomrule
       \end{tabular}
%      \end{sc}
%    \end{small}
  \end{center}
%  \vskip -0.1in
\end{table*}

\subsection{Hopfield Refinement Term $E^{\mathrm{HN}}$: Prototype Retrieval and Sharpening}
\label{sec:CDAT-hopfield}

% Our Mo-vMF analysis suggests that self-attention can be viewed as dynamically modeling directional similarity among inputs on a spherical manifold. Its updates correspond to descending an energy landscape toward posterior-weighted semantic prototypes, yielding a form of continuous, input-conditioned self-organization. This perspective naturally connects to classical associative memory models, most notably Hopfield networks~\citep{hopfield1982neural}.

% A Hopfield network explicitly stores a finite set of memory patterns and defines an energy function whose minima act as attractors; the system state evolves to converge to one of these stored memories~\citep{NIPS2016_eaae339c}. While Hopfield memories are fixed a priori—unlike the data-induced, adaptive prototypes in our model—both frameworks share the same core principle: state updates are governed by an underlying energy potential.
Mo--vMF analysis suggests that self-attention can be viewed as dynamically modeling directional similarity among inputs on a spherical manifold, yielding a form of continuous, input-driven self-organization. Furthermore, we introduce a Hopfield energy $E^{\mathrm{HN}}$ to enforce global structural consistency. Unlike the adaptive prototypes in Mo--vMF, this term anchors the dynamics to a set of fixed, learnable memory patterns, providing stable attractors for the trajectory~\citep{hopfield1982neural,NIPS2016_eaae339c}. Concretely, a modern Hopfield-type energy can be written by
\begin{equation}
\begin{split}
E^{\text{HN}}
= -\sum_{B=1}^{N}\sum_{\mu=1}^{P}
\text{ReLU}\!\left(\sum_{j=1}^{D}\xi_{\mu j}\, g_{jB}\right),
\end{split}
\end{equation}
where $\{\xi_\mu\}_{\mu=1}^{P} \in \mathbb{R}^{P\times D}$ denote learnable global stored memory vectors. From a dynamical systems perspective, this energy landscape exhibits $P$ local minima, each corresponding to a stored memory pattern. Overall, $E^{\mathrm{HN}}$ acts as a late-stage refinement mechanism: it performs prototype-style retrieval and sharpening, progressively dominating the trajectory as the system approaches stable attractors.

\begin{table*}[t]
  \caption{%Performance on Yelp, Amazon and T-Finance datasets with different training ratios.  Mean and standard deviation over 5 runs with different train/dev/test splits are reported (standard deviations are only included if they are available in the prior work). Best results are in \textbf{bold} and second-best results are underlined.
  Performance on Yelp, Amazon, and T-Finance across different training ratios. Means and standard deviations (when available from prior work) are reported over five runs with varying train/dev/test splits. Best results are in bold; second-best are underlined.}
  \label{tab:experiment_results_fixed}
  \begin{center}
%    \begin{\small}
%      \begin{sc}
        \begin{tabular}{Clcccccccc}
    \toprule
    \textbf{Metric}&Datasets & Split & CAREGNN & PC-GNN & BWGNN & GT & ET & UniGAD & \textbf{CDAT} \\
    % \midrule
    % \multicolumn{9}{l}{\textbf{MF1}} \\
    \midrule
    \multirow{8}{*}{\textbf{MF1}} 
    &Yelp & 1\% 
     & $62.1{\pm}1.3$ & $59.8{\pm}1.4$ & $61.1{\pm}0.4$ & $61.7{\pm}0.4$ & \underline{$62.7{\pm}1.9$} & -- & \textbf{63.0${\pm}$1.4} \\
    
    &Yelp & 40\% 
    & $63.3{\pm}0.9$ & $63.0{\pm}2.3$ & \textbf{71.0${\pm}$0.9} & $68.7{\pm}0.4$ & $69.5{\pm}0.2$ & -- & \underline{$70.5{\pm}0.1$} \\
    
    &Yelp & 70\% 
     & -- & -- & -- & -- & \underline{70.3${\pm}$0.4} & 70.2 & \textbf{71.5${\pm}$0.2} \\
     \cmidrule{2-10}
    
    &Amazon & 1\% 
     & $68.7{\pm}1.6$ & $79.8{\pm}5.6$ & \underline{$90.9{\pm}0.7$} & $88.6{\pm}0.5$ & $89.3{\pm}0.7$ & -- & \textbf{91.2${\pm}$0.9} \\
    
    &Amazon & 40\% 
    & $86.3{\pm}1.7$ & $85.0{\pm}0.7$ & \textbf{92.2${\pm}$0.4} & $91.7{\pm}0.8$ & $90.4{\pm}1.0$ & -- & \underline{$92.1{\pm}0.6$} \\
    
    &Amazon & 70\% 
     & -- & -- & -- & -- & \textbf{92${\pm}$0.3} & 91.3 & \underline{91.9${\pm}$0.1} \\
     \cmidrule{2-10}
    
   & T-Finance & 1\% 
    & 73.3 & 62.0 & \underline{86.8} & 81.5 & $85.1{\pm}1.0$ & -- & \textbf{87.4${\pm}$1.1} \\
    
   & T-Finance & 40\% 
    & 77.5 & 63.1 & 86.8 & 83.6 & $88.2{\pm}1.0$ & \underline{89.75} & \textbf{90.6${\pm}$0.6} \\
    % \midrule
    
    % \multicolumn{9}{l}{\textbf{AUC}} \\
    \midrule
   \multirow{8}{*}{\textbf{AUC}}& Yelp & 1\% 
    & \underline{$75.0{\pm}3.8$} & \textbf{75.4${\pm}$0.9} & $72.0{\pm}0.5$ & $72.5{\pm}0.6$ & $72.9{\pm}1.3$ & -- & $74.3{\pm}0.3$ \\
    
    &Yelp & 40\% 
     & $76.1{\pm}2.9$ & $79.8{\pm}0.1$ & \underline{$84.0{\pm}0.9$} & $81.9{\pm}0.5$ & $83.6{\pm}2.8$ & -- & \textbf{84.4${\pm}$0.2} \\
    
    &Yelp & 70\% 
     & -- & -- & -- & -- & 81.7${\pm}$1 & \underline{86.23} & \textbf{91.9${\pm}$0.3} \\
    \cmidrule{2-10}
    
    &Amazon & 1\% 
    & $88.6{\pm}3.5$ & $90.4{\pm}2.0$ & $89.4{\pm}0.3$ & $89.0{\pm}1.2$ & \underline{$91.9{\pm}1.0$} & -- & \textbf{94.5${\pm}$1.1} \\
    
    &Amazon & 40\% 
     & $90.5{\pm}1.6$ & $95.8{\pm}0.1$ & \textbf{98.0${\pm}$0.4} & $95.4{\pm}0.6$ & $95.7{\pm}2.0$ & -- & \underline{$97.2{\pm}0.2$} \\
    
    &Amazon & 70\% 
     & -- & -- & -- & -- & \underline{97.1${\pm}$2.1} & \textbf{97.84} & $93.0{\pm}5.0$ \\
    \cmidrule{2-10}
    
    &T-Finance & 1\% 
     & 90.5 & 90.7 & 91.1 & 90.0 & \underline{$92.2{\pm}1.1$} & -- & \textbf{95.91${\pm}$0.4} \\
    
    &T-Finance & 40\% 
     & 92.1 & 91.2 & 94.3 & 88.2 & $95.0{\pm}3.0$ & \underline{96.49} & \textbf{97.84${\pm}$1.1} \\
    
    \bottomrule
    \end{tabular}
%      \end{sc}
%    \end{small}
  \end{center}
%  \vskip -0.1in
\end{table*}

\section{Empirical Evaluation}
Although the CDAT formulation is general and applies to any tokenized relational data, we focus on the graph domain as a principled and challenging testbed. Graph representation learning requires simultaneously preserving local structural fidelity, integrating global context, and remaining robust to noisy features and perturbations~\citep{hamilton2017representation, wu2020comprehensive}. Classic message-passing models such as GCNs ~\citep{kipf2017semisupervised} and their variants usually suffer from over-smoothing at depth, while attention-based graph models (\emph{e.g.}, GAT~\citep{veličković2018graph}) may over-globalize interactions or face scalability issues. These challenges make graphs a natural benchmark for studying whether energy-based attractor dynamics can stabilize learning and improve discriminability. 

In following sections, we empirically validate the CDAT's benefits on standard graph learning tasks. We conduct extensive experiments on both graph-level classification and graph anomaly detection tasks to comprehensively evaluate the effectiveness of the proposed CDAT. Overall, CDAT demonstrates consistent and competitive performance across diverse datasets. The code is available at \url{https://github.com/Angelov1vil/CDAT}.

\subsection{Graph Classification with \textbf{CDAT}}

We compare CDAT with the current state of the art approaches for the mentioned datasets, which include GRDL~\citep{wang2024graph}, GRANOLA~\citep{eliasof2024granola}, ET~\citep{hoover2023energy}, DSGCN~\citep{balcilar2020bridging}, HGP-SL~\citep{zhang2021hierarchical}. Additionally, approaches~\citep{yang2022new,orsini2015graph,ranjan2020asap,zhao2019learning,cai2021graphnorm,nguyen2022universal,bianchi2020spectral}, which are close to the baselines, are included to further contrast the performance of our model. The results are summarized in Table~\ref{tab:classification}.

Across seven widely used datasets, CDAT achieves state-of-the-art performance on 5 out of 7 datasets. 
Specifically, CDAT yields significant absolute accuracy gains over ET, improving by 3.4\% on NCI1, 3.8\% on NCI109, and 2.2\% on MUTAG. 
On the challenging DD dataset, CDAT not only surpasses ET by 2.6\% but also demonstrates overwhelming dominance over classical pooling methods (\emph{e.g.}, exceeding HGP-SL by $>$13\%), highlighting its capability to capture complex structural dependencies. 
While ET remains competitive on PROTEINS, CDAT's superior performance across the majority of benchmarks confirms that our attractor dynamics provide a more expressive and stable inductive bias than naive energy descent.

Notably, the gains are most pronounced on structurally complex datasets such as DD (284 avg. nodes) and the NCI family, where heterogeneous relational interactions demand richer aggregation beyond what plain energy minimization can offer. This suggests that the topology-conditioned excitation–inhibition modulation is especially beneficial when graph structure is dense and varied. On smaller or simpler benchmarks such as PROTEINS, the margin narrows, which we attribute to reduced structural heterogeneity rather than a fundamental limitation of the proposed dynamics.

\begin{table}[t]
  \centering
  \caption{Ablation on the bottleneck rank $r$ in $W = P^\top Q P$ on DD
  and NCI1. Accuracy (\%, mean $\pm$ std over 100 runs of 10-fold CV) is
  reported. ``Without $W$'' removes the mutual-excitation term;
  ``Full-rank $W$'' uses an unconstrained $W \in \mathbb{R}^{D\times D}$
  without symmetric or low-rank constraints.}
  \label{tab:ablation_rank}
  \begin{center}
      \begin{tabular}{lcc}
    \toprule
    Config & DD & NCI1 \\
    \midrule
    Without $W$       & $83.1 \pm 2.1$           & $82.8 \pm 0.5$ \\
    $r = 1$           & $82.2 \pm 4.8$           & $87.6 \pm 0.3$ \\
    $r = 2$           & $89.2 \pm 4.1$           & $90.2 \pm 1.1$ \\
    $r = 4$ (default) & $\mathbf{98.5 \pm 0.6}$  & $\mathbf{93.5 \pm 0.2}$ \\
    $r = 8$           & $97.6 \pm 1.3$           & $77.0 \pm 0.7$ \\
    $r = 16$          & $97.8 \pm 0.2$           & $77.5 \pm 1.2$ \\
    Full rank $W$        & $98.0 \pm 0.1$           & $82.9 \pm 0.4$ \\
    \bottomrule
      \end{tabular}
  \end{center}
\end{table}

\subsection{Graph Anomaly Detection with \textbf{CDAT}}
%We further assess \textbf{CDAT} on graph anomaly detection tasks, following established experimental protocols on Yelp, Amazon, and T-Finance datasets. Both MF1 and AUC metrics are reported under different labeled anomaly ratios. The full results are presented in Table~\ref{tab:experiment_results_fixed}.
We evaluate CDAT on graph anomaly detection following established protocols on Yelp, Amazon and T-Finance. Both MF1 and AUC are reported under varying labeled anomaly ratios. Full results are given in Table~\ref{tab:experiment_results_fixed}.
%\textbf{CDAT} establishes a new state of the art across benchmarks, achieving substantial gains over strong baselines and recent methods~\citep{lin2024unigad,dou2020enhancing,liu2021pick,tang2022rethinking,dwivedi2020generalization,hoover2023energy}. Crucially, in the challenging low-label regime ($1\%$), \textbf{CDAT} achieves remarkable gains over ET, improving AUC by 2.6\% on Amazon and 3.7\% on T-Finance. Furthermore, compared to UNIGAD, \textbf{CDAT} demonstrates superior robustness in complex environments: on the dense Yelp dataset ($70\%$ split), it surpasses UNIGAD by a massive 5.6\% in AUC ($91.9\%$ vs. $86.23\%$) while also maintaining a higher MF1 score. Similar dominance is observed on T-Finance, where \textbf{CDAT} consistently outperforms UNIGAD across both metrics.
CDAT achieves state-of-the-art performance across benchmarks, with substantial gains over strong baselines and recent methods~\citep{lin2024unigad,dou2020enhancing,liu2021pick,tang2022rethinking,dwivedi2020generalization,hoover2023energy}. Notably, in the challenging low-label regime ($1\%$ split), CDAT outperforms ET with AUC improvements of 2.6\% on Amazon and 3.7\% on T-Finance. Moreover, CDAT shows superior robustness over UNIGAD in complex environments. On the dense Yelp dataset ($70\%$ split), it surpasses UNIGAD by 5.6\% in AUC (91.9\% vs. 86.23\%) while also achieving a higher MF1 score. Similar dominance holds on T-Finance, where CDAT consistently outperforms UNIGAD across both metrics.
These results highlight the advantage of modeling anomaly detection as a
structured dynamical evolution process rather than a purely static scoring
problem. The combination of mutual excitation and self-inhibition allows CDAT
to better separate anomalous patterns from normal graph structures, leading to
improved detection accuracy under challenging supervision constraints.

\begin{table*}[t]
  \centering
  \caption{Wall-clock and parameter comparison on NCI109. Percentages in
  parentheses denote the relative overhead of CDAT over ET.}
  \label{tab:runtime}
  \begin{center}
  \begin{tabular}{lccc}
    \toprule
    Metric & GT & ET & CDAT (ours) \\
    \midrule
    Params               & 1.074\,M           & 1.058\,M           & 1.060\,M         \\
    Inference (ms/batch) & $1.46 \pm 0.05$    & $6.35 \pm 0.60$    & $6.50 \pm 0.90$   \\
    Training (s/epoch)   & $0.753 \pm 0.097$  & $2.034 \pm 0.186$  & $2.046 \pm 0.157$ \\
    \bottomrule
  \end{tabular}
  \end{center}
\end{table*}

\subsection{Ablation on Topology-conditioned Mutual Excitation $W$}
\label{sec:ablation_W}
To clarify the functional role of the coupling operator $W$ and to
validate its symmetric low-rank parameterization in
\cref{eq:sym_lowrank_W}, we conduct an ablation on two
representative graph classification benchmarks (DD and NCI1).
Concretely, we probe how the bottleneck rank $r$ in $W = P^\top Q P$
governs the trade-off between expressiveness and generalization by
sweeping $r \in \{1, 2, 4, 8, 16\}$. To delimit the two extremes of
this spectrum, we further include a variant that removes $W\mathbf{x}$
entirely, corresponding to the zero-capacity limit,
and an unconstrained full-rank parameterization that drops both the
symmetric and the low-rank constraints, corresponding to the
maximum-capacity limit. All results are reported in Table~\ref{tab:ablation_rank}, from which
three observations emerge along this capacity spectrum. At the
zero-capacity limit, removing $W\mathbf{x}$ entirely drops accuracy
from $98.5\%$ to $83.1\%$ on DD and from $93.5\%$ to $82.8\%$ on NCI1,
confirming that the feature-coordinate couplings contribute structural
priors that the semantic energies alone cannot recover. At the
maximum-capacity limit, by contrast, accuracy on NCI1 degrades sharply
once $r > 4$ and the unconstrained full-rank parameterization
($\sim$\,251K extra parameters) collapses to $82.9\%$; this indicates
that excess capacity introduces spurious feature-dimension couplings
prone to overfitting, empirically validating $W = P^\top Q P$ as a
principled regularizer that suppresses such modes. Between these two
extremes, $r = 4$ emerges as a robust default, delivering the best
accuracy on both benchmarks with only $2{,}020$ parameters ($0.19\%$
of the total model budget), and is therefore adopted throughout our
experiments.

Notably, the full-rank configuration remains competitive on DD ($98.0\%$) but
degrades substantially on NCI1 ($82.9\%$). This shows that enforcing
the symmetry of $W$ does not bottleneck performance; rather, the symmetric
low-rank parameterization serves as a principled inductive bias that
stabilizes training while preserving the dissipative structure required
by Theorem~\ref{thm:dissipation_certificate}.

\subsection{Inference Cost and Runtime Comparison}

To quantify the computational overhead introduced by our formulation, we
conduct a controlled wall-clock comparison on the NCI109 dataset against ET~\citep{hoover2023energy} and GT ~\citep{dwivedi2020generalization}
baselines, keeping all other settings (batch size, hardware, precision,
optimizer) identical. Results are averaged over 5 runs and reported
as mean~$\pm$~std in Table~\ref{tab:runtime}. 
We observe that CDAT and ET exhibit essentially identical computational cost: the additional parameters
introduced by our design account for only $+0.19\%$, and inference latency
increases by just $+2.4\%$. This confirms that our modifications do not
meaningfully alter the computational profile of the energy-based paradigm.
The feed-forward GT baseline achieves lower latency through single-pass
execution, but at the cost of substantially weaker task performance
(see Table~\ref{tab:experiment_results_fixed}). 
This latency is intrinsic to the iterative relaxation paradigm
underlying energy-based models and is not introduced by CDAT.

\section{Limitations and Future Work}
%In this paper, we exploited \textbf{CDAT}, an energy-based framework that integrates CANN-inspired dynamics to stabilize inference trajectories. While \textbf{CDAT} achieves state-of-the-art performance in graph classification and anomaly detection by suppressing spurious local minima, its reliance on iterative state unfolding incurs higher inference costs compared to standard feed-forward Transformers. To address this, future work will focus on learning a more expressive control interface, such as state-dependent schedules for the damping parameter $\omega$, to optimize the trade-off between convergence speed and robustness. Furthermore, we plan to extend the framework with richer semantic energy landscapes—beyond standard $E_{\mathrm{vMF}}$ and Hopfield energies—by incorporating multi-prototype hyperspherical energies and task-adaptive energy families. Finally, recognizing CDAT’s structural flexibility, we plan to deploy it as a specialized substructure within larger hierarchical architectures, thereby verifying its utility as a versatile component like~\citep{cao2025memory}.

In this paper, we present CDAT, an energy-based framework incorporating CANN-inspired dynamics to stabilize inference trajectories. While CDAT achieves state-of-the-art graph classification and anomaly detection by suppressing spurious local minima, its iterative state unfolding incurs higher inference costs than standard feed-forward Transformers.
To address this, future work will focus on learning a more expressive control interface, such as state-dependent schedules for damping parameter $\omega$, to optimize the trade-off between convergence speed and robustness. 
Furthermore, we plan to extend the framework with richer semantic energy landscapes—beyond standard Mo-vMF and Hopfield energies by incorporating multi-prototype hyperspherical energies and task-adaptive energy families. 
Finally, given CDAT’s structural flexibility, we intend to deploy it as a specialized substructure within larger hierarchical architectures, validating its versatility akin to~\citep{cao2025memory,10903668}.

\section*{Acknowledgments}
This work is supported by the Fundamental and Interdisciplinary Disciplines Breakthrough Plan of the Ministry of Education of China (No. JYB2025XDXM101), the National Natural Science Foundation of China (No. 62272374), the Natural Science Foundation of Shaanxi Province (No.2024JC-JCQN-62), the State Key Laboratory of Communication Content Cognition under Grant No. A202502, and the Key Research and Development Project in Shaanxi Province (No. 2023GXLH-024).

\section*{Impact Statement}

This paper presents work whose goal is to advance the field of Machine
Learning. There are many potential societal consequences of our work, none
which we feel must be specifically highlighted here.

\nocite{langley00}

\bibliography{cite/cite}

@InProceedings{pmlr-v139-davis21a,
  title     = {Catformer: Designing stable Transformers via sensitivity analysis},
  author    = {Davis, Jared Q and Gu, Albert and Choromanski, Krzysztof and Dao, Tri and Re, Christopher and Finn, Chelsea and Liang, Percy},
  booktitle = {International Conference on Machine Learning},
  pages     = {2489--2499},
  year      = {2021},
  editor    = {Meila, Marina and Zhang, Tong},
  volume    = {139},
  series    = {Proceedings of Machine Learning Research},
  publisher = {PMLR},
  url       = {https://proceedings.mlr.press/v139/davis21a.html}
}

@inproceedings{NIPS2016_eaae339c,
 author = {Krotov, Dmitry and Hopfield, John J.},
 booktitle = {Advances in Neural Information Processing Systems},
 editor = {D. Lee and M. Sugiyama and U. Luxburg and I. Guyon and R. Garnett},
 publisher = {Curran Associates, Inc.},
 title = {Dense Associative Memory for Pattern Recognition},
 volume = {29},
 year = {2016}
}

@inproceedings{
krotov2021large,
title={Large Associative Memory Problem in Neurobiology and Machine Learning},
author={Dmitry Krotov and John J. Hopfield},
booktitle={International Conference on Learning Representations},
year={2021},
url={https://openreview.net/forum?id=X4y_10OX-hX}
}

@book{hinton2014parallel,
  title={Parallel Models of Associative Memory: Updated Edition},
  author={Hinton, G.E. and Anderson, J.A.},
  isbn={9781317785200},
  url={https://books.google.com.sg/books?id=4w3sAgAAQBAJ},
  year={2014},
  publisher={Taylor \& Francis}
}

@article{tsodyks1995associative,
  title={Associative Memory and Hippocampal Place Cells},
  author={Tsodyks, Misha and Sejnowski, Terrence},
  journal={International Journal of Neural Systems},
  volume={6},
  pages={81--86},
  year={1995},
  publisher={World Scientific}
}

@article{rolls2013mechanisms,
  title={The Mechanisms for Pattern Completion and Pattern Separation in the Hippocampus},
  author={Rolls, Edmund T},
  journal={Frontiers in Systems Neuroscience},
  volume={7},
  pages={74},
  year={2013},
  publisher={Frontiers Media SA}
}

@article{hopfield1982neural,
  title={Neural networks and physical systems with emergent collective computational abilities.},
  author={Hopfield, John J},
  journal={Proceedings of the National Academy of Sciences},
  volume={79},
  number={8},
  pages={2554--2558},
  year={1982}
}

@inproceedings{
ramsauer2021hopfield,
title={Hopfield Networks is All You Need},
author={Hubert Ramsauer and Bernhard Sch{\"a}fl and Johannes Lehner and Philipp Seidl and Michael Widrich and Lukas Gruber and Markus Holzleitner and Thomas Adler and David Kreil and Michael K Kopp and G{\"u}nter Klambauer and Johannes Brandstetter and Sepp Hochreiter},
booktitle={International Conference on Learning Representations},
year={2021},
url={https://openreview.net/forum?id=tL89RnzIiCd}
}

@inproceedings{NEURIPS2021_8171ac2c,
 author = {Bricken, Trenton and Pehlevan, Cengiz},
 booktitle = {Advances in Neural Information Processing Systems},
 editor = {M. Ranzato and A. Beygelzimer and Y. Dauphin and P.S. Liang and J. Wortman Vaughan},
 pages = {15301--15315},
 publisher = {Curran Associates, Inc.},
 title = {Attention Approximates Sparse Distributed Memory},
 volume = {34},
 year = {2021}
}

@inproceedings{hoover2023energy,
 author = {Hoover, Benjamin and Liang, Yuchen and Pham, Bao and Panda, Rameswar and Strobelt, Hendrik and Chau, Duen Horng and Zaki, Mohammed and Krotov, Dmitry},
 booktitle = {Advances in Neural Information Processing Systems},
 editor = {A. Oh and T. Naumann and A. Globerson and K. Saenko and M. Hardt and S. Levine},
 pages = {27532--27559},
 publisher = {Curran Associates, Inc.},
 title = {Energy Transformer},
 volume = {36},
 year = {2023}
}

@book{amit1989model,
  title={Modeling brain function: The world of attractor neural networks},
  author={Amit, Daniel J},
  year={1989},
  publisher={Cambridge university press}
}

@article{amari1977dynamics,
  title={Dynamics of pattern formation in lateral-inhibition type neural fields},
  author={Amari, Shun-ichi},
  journal={Biological Cybernetics},
  volume={27},
  number={2},
  pages={77--87},
  year={1977},
  publisher={Springer}
}

@article{samsonovich1997path,
  title={Path integration and cognitive mapping in a continuous attractor neural network model},
  author={Samsonovich, Alexei and McNaughton, Bruce L},
  journal={Journal of Neuroscience},
  volume={17},
  number={15},
  pages={5900--5920},
  year={1997},
  publisher={Society for Neuroscience}
}

@article{hamilton2017representation,
  title={Representation learning on graphs: Methods and applications},
  author={Hamilton, William L and Ying, Rex and Leskovec, Jure},
  journal={arXiv preprint arXiv:1709.05584},
  year={2017}
}

@article{wu2020comprehensive,
  title={A comprehensive survey on graph neural networks},
  author={Wu, Zonghan and Pan, Shirui and Chen, Fengwen and Long, Guodong and Zhang, Chengqi and Yu, Philip S},
  journal={IEEE transactions on Neural Networks and Learning Systems},
  volume={32},
  number={1},
  pages={4--24},
  year={2021},
  publisher={IEEE}
}

@inproceedings{millidge2022universal,
  title={Universal hopfield networks: A general framework for single-shot associative memory models},
  author={Millidge, Beren and Salvatori, Tommaso and Song, Yuhang and Lukasiewicz, Thomas and Bogacz, Rafal},
  booktitle={International Conference on Machine Learning},
  pages={15561--15583},
  year={2022},
  organization={PMLR}
}

@article{wu2016continuous,
  title={Continuous attractor neural networks: candidate of a canonical model for neural information representation},
  author={Wu, Si and Wong, KY Michael and Fung, CC Alan and Mi, Yuanyuan and Zhang, Wenhao},
  journal={F1000Research},
  volume={5},
  pages={F1000--Faculty},
  year={2016}
}

@inproceedings{NIPS2014_57c76ace,
 author = {Mi, Yuanyuan and Fung, C. C. Alan and Wong, K. Y. Michael and Wu, Si},
 booktitle = {Advances in Neural Information Processing Systems},
 editor = {Z. Ghahramani and M. Welling and C. Cortes and N. Lawrence and K.Q. Weinberger},
 publisher = {Curran Associates, Inc.},
 title = {Spike Frequency Adaptation Implements Anticipative Tracking in Continuous Attractor Neural Networks},
 volume = {27},
 year = {2014}
}

@inproceedings{
kipf2017semisupervised,
title={Semi-Supervised Classification with Graph Convolutional Networks},
author={Thomas N. Kipf and Max Welling},
booktitle={International Conference on Learning Representations},
year={2017},
url={https://openreview.net/forum?id=SJU4ayYgl}
}

@inproceedings{
veličković2018graph,
title={Graph Attention Networks},
author={Petar Veličković and Guillem Cucurull and Arantxa Casanova and Adriana Romero and Pietro Liò and Yoshua Bengio},
booktitle={International Conference on Learning Representations},
year={2018},
url={https://openreview.net/forum?id=rJXMpikCZ},
}

@inproceedings{saha2023end,
  title={End-to-end differentiable clustering with associative memories},
  author={Saha, Bishwajit and Krotov, Dmitry and Zaki, Mohammed J and Ram, Parikshit},
  booktitle={International Conference on Machine Learning},
  pages={29649--29670},
  year={2023},
  organization={PMLR}
}

@inproceedings{
schaeffer2023associative,
title={Associative Memory Under the Probabilistic Lens: Improved Transformers \& Dynamic Memory Creation},
author={Rylan Schaeffer and Mikail Khona and Nika Zahedi and Ila R Fiete and Andrey Gromov and Sanmi Koyejo},
booktitle={Associative Memory {\&} Hopfield Networks in 2023},
year={2023},
url={https://openreview.net/forum?id=lO61aZlteS}
}

@inproceedings{sharma2022content,
  title={Content addressable memory without catastrophic forgetting by heteroassociation with a fixed scaffold},
  author={Sharma, Sugandha and Chandra, Sarthak and Fiete, Ila},
  booktitle={International Conference on Machine Learning},
  pages={19658--19682},
  year={2022},
  organization={PMLR}
}

@article{lecun2002gradient,
  title={Gradient-based learning applied to document recognition},
  author={LeCun, Yann and Bottou, L{\'e}on and Bengio, Yoshua and Haffner, Patrick},
  journal={Proceedings of the IEEE},
  volume={86},
  number={11},
  pages={2278--2324},
  year={2002},
  publisher={Ieee}
}

@misc{tang2021remarkpaperkrotovhopfield,
      title={A remark on a paper of Krotov and Hopfield [arXiv:2008.06996]}, 
      author={Fei Tang and Michael Kopp},
      year={2021},
      eprint={2105.15034},
      archivePrefix={arXiv},
      primaryClass={q-bio.NC},
      url={https://arxiv.org/abs/2105.15034}, 
}

@article{DBLP:journals/corr/abs-2107-06446,
  author       = {Dmitry Krotov},
  title        = {Hierarchical Associative Memory},
  journal      = {CoRR},
  volume       = {abs/2107.06446},
  year         = {2021},
  url          = {https://arxiv.org/abs/2107.06446},
  eprinttype    = {arXiv},
  eprint       = {2107.06446},
}

@article{demircigil2017model,
  title={On a model of associative memory with huge storage capacity},
  author={Demircigil, Mete and Heusel, Judith and L{\"o}we, Matthias and Upgang, Sven and Vermet, Franck},
  journal={Journal of Statistical Physics},
  volume={168},
  number={2},
  pages={288--299},
  year={2017},
  publisher={Springer}
}

@article{amit1985storing,
  title={Storing infinite numbers of patterns in a spin-glass model of neural networks},
  author={Amit, Daniel J and Gutfreund, Hanoch and Sompolinsky, Haim},
  journal={Physical Review Letters},
  volume={55},
  number={14},
  pages={1530},
  year={1985},
  publisher={APS}
}

@article{lucibello2024exponential,
  title={Exponential capacity of dense associative memories},
  author={Lucibello, Carlo and M{\'e}zard, Marc},
  journal={Physical Review Letters},
  volume={132},
  number={7},
  pages={077301},
  year={2024},
  publisher={APS}
}

@inproceedings{NEURIPS2024_29ff36c8,
 author = {Hoover, Benjamin and Chau, Duen Horng and Strobelt, Hendrik and Ram, Parikshit and Krotov, Dmitry},
 booktitle = {Advances in Neural Information Processing Systems},
 doi = {10.52202/079017-0742},
 editor = {A. Globerson and L. Mackey and D. Belgrave and A. Fan and U. Paquet and J. Tomczak and C. Zhang},
 pages = {23549--23576},
 publisher = {Curran Associates, Inc.},
 title = {Dense Associative Memory Through the Lens of Random Features},
 volume = {37},
 year = {2024}
}

@inproceedings{du2022learning,
  title={Learning iterative reasoning through energy minimization},
  author={Du, Yilun and Li, Shuang and Tenenbaum, Joshua and Mordatch, Igor},
  booktitle={International Conference on Machine Learning},
  pages={5570--5582},
  year={2022},
  organization={PMLR}
}

@ARTICLE{gutkin2014spike,
AUTHOR = {Gutkin, B.  and Zeldenrust, F. },
TITLE   = {{S}pike frequency adaptation},
YEAR    = {2014},
JOURNAL = {Scholarpedia},
VOLUME  = {9},
NUMBER  = {2},
PAGES   = {30643},
DOI     = {10.4249/scholarpedia.30643},
NOTE    = {revision \#143322}
}

@article{seeholzer2019stability,
  title={Stability of working memory in continuous attractor networks under the control of short-term plasticity},
  author={Seeholzer, Alexander and Deger, Moritz and Gerstner, Wulfram},
  journal={PLoS Computational Biology},
  volume={15},
  number={4},
  pages={e1006928},
  year={2019},
  publisher={Public Library of Science San Francisco, CA USA}
}

@inproceedings{
hu2022lora,
title={Lo{RA}: Low-Rank Adaptation of Large Language Models},
author={Edward J Hu and yelong shen and Phillip Wallis and Zeyuan Allen-Zhu and Yuanzhi Li and Shean Wang and Lu Wang and Weizhu Chen},
booktitle={International Conference on Learning Representations},
year={2022},
url={https://openreview.net/forum?id=nZeVKeeFYf9}
}

@inproceedings{NEURIPS2023_57bc0a85,
 author = {Hu, Jerry Yao-Chieh and Yang, Donglin and Wu, Dennis and Xu, Chenwei and Chen, Bo-Yu and Liu, Han},
 booktitle = {Advances in Neural Information Processing Systems},
 editor = {A. Oh and T. Naumann and A. Globerson and K. Saenko and M. Hardt and S. Levine},
 pages = {27594--27608},
 publisher = {Curran Associates, Inc.},
 title = {On Sparse Modern Hopfield Model},
 volume = {36},
 year = {2023}
}

@InProceedings{Sun_2025_CVPR,
    author    = {Sun, Yuwei and Ochiai, Hideya and Wu, Zhirong and Lin, Stephen and Kanai, Ryota},
    title     = {Associative Transformer},
    booktitle = {Proceedings of the IEEE/CVF Conference on Computer Vision and Pattern Recognition },
    month     = {June},
    year      = {2025},
    pages     = {4518-4527}
}

@inproceedings{
wu2025incontext,
title={In-Context Learning as Conditioned Associative Memory Retrieval},
author={Weimin Wu and Teng-Yun Hsiao and Jerry Yao-Chieh Hu and Wenxin Zhang and Han Liu},
booktitle={International Conference on Machine Learning},
year={2025},
url={https://openreview.net/forum?id=Zup6F3MwQO}
}

@article{chandra2025episodic,
  title={Episodic and associative memory from spatial scaffolds in the hippocampus},
  author={Chandra, Sarthak and Sharma, Sugandha and Chaudhuri, Rishidev and Fiete, Ila},
  journal={Nature},
  volume={638},
  number={8051},
  pages={739--751},
  year={2025},
  publisher={Nature Publishing Group UK London}
}

@inproceedings{
cao2025memory,
title={Memory Decoder: A Pretrained, Plug-and-Play Memory for Large Language Models},
author={Jiaqi Cao and Jiarui Wang and Rubin Wei and Qipeng Guo and Kai Chen and Bowen Zhou and Zhouhan Lin},
booktitle={Advances in Neural Information Processing Systems},
year={2025},
url={https://openreview.net/forum?id=ARJpQtLXfe}
}

@article{hu2024outlier,
  title={Outlier-efficient hopfield layers for large transformer-based models},
  author={Hu, Jerry Yao-Chieh and Chang, Pei-Hsuan and Luo, Robin and Chen, Hong-Yu and Li, Weijian and Wang, Wei-Po and Liu, Han},
  journal={arXiv preprint arXiv:2404.03828},
  year={2024}
}

@article{morris2020tudataset,
  title={Tudataset: A collection of benchmark datasets for learning with graphs},
  author={Morris, Christopher and Kriege, Nils M and Bause, Franka and Kersting, Kristian and Mutzel, Petra and Neumann, Marion},
  journal={arXiv preprint arXiv:2007.08663},
  year={2020}
}

@inproceedings{tang2022rethinking,
  title={Rethinking graph neural networks for anomaly detection},
  author={Tang, Jianheng and Li, Jiajin and Gao, Ziqi and Li, Jia},
  booktitle={International Conference on Machine Learning},
  pages={21076--21089},
  year={2022},
  organization={PMLR}
}

@inproceedings{yang2022new,
  title={A new perspective on the effects of spectrum in graph neural networks},
  author={Yang, Mingqi and Shen, Yanming and Li, Rui and Qi, Heng and Zhang, Qiang and Yin, Baocai},
  booktitle={International Conference on Machine Learning},
  pages={25261--25279},
  year={2022},
  organization={PMLR}
}

@inproceedings{orsini2015graph,
  title={Graph invariant kernels},
  author={Orsini, Francesco and Frasconi, Paolo and De Raedt, Luc and Yang, Qiang and Wooldridge, Michael},
  booktitle={Proceedings of the twenty-fourth international joint conference on artificial intelligence},
  volume={2015},
  pages={3756--3762},
  year={2015},
  organization={IJCAI-INT JOINT CONF ARTIF INTELL}
}

@inproceedings{ranjan2020asap,
  title={Asap: Adaptive structure aware pooling for learning hierarchical graph representations},
  author={Ranjan, Ekagra and Sanyal, Soumya and Talukdar, Partha},
  booktitle={Proceedings of the AAAI conference on artificial intelligence},
  volume={34},
  number={04},
  pages={5470--5477},
  year={2020}
}

@inproceedings{zhao2019learning,
  author={Qi Zhao and Yusu Wang},
  title={Learning metrics for persistence-based summaries and applications for graph classification},
  year={2019},
  cdate={1546300800000},
  pages={9855-9866},
  booktitle={Advances in Neural Information Processing Systems},
}

@inproceedings{
wang2024graph,
title={Graph Classification via Reference Distribution Learning: Theory and Practice},
author={Zixiao Wang and Jicong Fan},
booktitle={Advances in Neural Information Processing Systems},
year={2024},
url={https://openreview.net/forum?id=1zVinhehks}
}

@article{balcilar2020bridging,
  title={Bridging the gap between spectral and spatial domains in graph neural networks},
  author={Balcilar, Muhammet and Renton, Guillaume and H{\'e}roux, Pierre and Gauzere, Benoit and Adam, Sebastien and Honeine, Paul},
  journal={arXiv preprint arXiv:2003.11702},
  year={2020}
}

@inproceedings{cai2021graphnorm,
  title={Graphnorm: A principled approach to accelerating graph neural network training},
  author={Cai, Tianle and Luo, Shengjie and Xu, Keyulu and He, Di and Liu, Tie-yan and Wang, Liwei},
  booktitle={International Conference on Machine Learning},
  pages={1204--1215},
  year={2021},
  organization={PMLR}
}

@article{zhang2021hierarchical,
  title={Hierarchical multi-view graph pooling with structure learning},
  author={Zhang, Zhen and Bu, Jiajun and Ester, Martin and Zhang, Jianfeng and Li, Zhao and Yao, Chengwei and Dai, Huifen and Yu, Zhi and Wang, Can},
  journal={IEEE Transactions on Knowledge and Data Engineering},
  volume={35},
  number={1},
  pages={545--559},
  year={2021},
  publisher={IEEE}
}

@inproceedings{nguyen2022universal,
  title={Universal graph transformer self-attention networks},
  author={Nguyen, Dai Quoc and Nguyen, Tu Dinh and Phung, Dinh},
  booktitle={Companion Proceedings of the Web Conference 2022},
  pages={193--196},
  year={2022}
}

@inproceedings{eliasof2024granola,
 author = {Eliasof, Moshe and Bevilacqua, Beatrice and Sch\"{o}nlieb, Carola-Bibiane and Maron, Haggai},
 booktitle = {Advances in Neural Information Processing Systems},
 doi = {10.52202/079017-2873},
 editor = {A. Globerson and L. Mackey and D. Belgrave and A. Fan and U. Paquet and J. Tomczak and C. Zhang},
 pages = {90514--90551},
 publisher = {Curran Associates, Inc.},
 title = {GRANOLA: Adaptive Normalization for Graph Neural Networks},
 volume = {37},
 year = {2024}
}

@inproceedings{bianchi2020spectral,
  title={Spectral clustering with graph neural networks for graph pooling},
  author={Bianchi, Filippo Maria and Grattarola, Daniele and Alippi, Cesare},
  booktitle={International Conference on Machine Learning},
  pages={874--883},
  year={2020},
  organization={PMLR}
}

@inproceedings{
lin2024unigad,
title={Uni{GAD}: Unifying Multi-level Graph Anomaly Detection},
author={Yiqing Lin and Jianheng Tang and Chenyi Zi and H. Vicky Zhao and Yuan Yao and Jia Li},
booktitle={Advances in Neural Information Processing Systems},
year={2024},
url={https://openreview.net/forum?id=sRILMnkkQd}
}

@inproceedings{dou2020enhancing,
  title={Enhancing graph neural network-based fraud detectors against camouflaged fraudsters},
  author={Dou, Yingtong and Liu, Zhiwei and Sun, Li and Deng, Yutong and Peng, Hao and Yu, Philip S},
  booktitle={Proceedings of the 29th ACM International Conference on Information \& Knowledge Management},
  pages={315--324},
  year={2020}
}

@inproceedings{liu2021pick,
  title={Pick and choose: a GNN-based imbalanced learning approach for fraud detection},
  author={Liu, Yang and Ao, Xiang and Qin, Zidi and Chi, Jianfeng and Feng, Jinghua and Yang, Hao and He, Qing},
  booktitle={Proceedings of the Web Conference 2021},
  pages={3168--3177},
  year={2021}
}

@article{dwivedi2020generalization,
  title={A generalization of transformer networks to graphs},
  author={Dwivedi, Vijay Prakash and Bresson, Xavier},
  journal={arXiv preprint arXiv:2012.09699},
  year={2020}
}

@ARTICLE{10903668,
  author={Zhuge, Mingchen and Liu, Haozhe and Faccio, Francesco and Ashley, Dylan R. and Csordás, Róbert and Gopalakrishnan, Anand and Hamdi, Abdullah and Hammoud, Hasan Abed Al Kader and Herrmann, Vincent and Irie, Kazuki and Kirsch, Louis and Li, Bing and Li, Guohao and Liu, Shuming and Mai, Jinjie and Piękos, Piotr and Ramesh, Aditya A. and Schlag, Imanol and Shi, Weimin and Stanić, Aleksandar and Wang, Wenyi and Wang, Yuhui and Xu, Mengmeng and Fan, Deng-Ping and Ghanem, Bernard and Schmidhuber, Jürgen},
  journal={Computational Visual Media}, 
  title={Mindstorms in Natural Language-Based Societies of Mind}, 
  year={2025},
  volume={11},
  number={1},
  pages={29-81},
  doi={10.26599/CVM.2025.9450460}}
\bibliographystyle{icml2026}

%%%%%%%%%%%%%%%%%%%%%%%%%%%%%%%%%%%%%%%%%%%%%%%%%%%%%%%%%%%%%%%%%%%%%%%%%%%%%%%
%%%%%%%%%%%%%%%%%%%%%%%%%%%%%%%%%%%%%%%%%%%%%%%%%%%%%%%%%%%%%%%%%%%%%%%%%%%%%%%
% APPENDIX
%%%%%%%%%%%%%%%%%%%%%%%%%%%%%%%%%%%%%%%%%%%%%%%%%%%%%%%%%%%%%%%%%%%%%%%%%%%%%%%
%%%%%%%%%%%%%%%%%%%%%%%%%%%%%%%%%%%%%%%%%%%%%%%%%%%%%%%%%%%%%%%%%%%%%%%%%%%%%%%
\newpage
\appendix
\onecolumn
\section{Notations Used in the Main Text and Appendices}
Table ~\ref{tab:notations} lists all the notations used in this paper.
\begin{table}[ht]
\centering
\caption{Notations used in this paper.}
\label{tab:notations}
\begin{tabular}{cl}
\toprule
\textbf{Notation} & \textbf{Description} \\
\midrule
$D$ & Dimension of the state/token feature space \\
$N$ & Number of tokens in the system \\
$\mathbf{x}$ & Vector representation of token\\
${x}_i$ & Each element of $\mathbf{x}$\\
$X$ & State matrix stacking token states, $X=[\mathbf{x}_1,\mathbf{x}_2,\cdots,\mathbf{x}_N]^\top \in \mathbb{R}^{N \times D} $  \\
$\mathbf{g}$ & LayerNorm output applied to a token vector $\mathbf{x}$ \\
${g}_i$ & Each element of $\mathbf{g}$ \\
$G$ & Normalized state matrix, $G=[\mathbf{g}_1,\cdots,\mathbf{g}_N]^\top \in \mathbb{R}^{N \times D}$ \\
$M$ & Jacobian of LayerNorm, $M(\mathbf{x}(t))=\nabla_{\mathbf{x}}\mathbf{g}(\mathbf{x}(t))$ \\
\midrule
$\mathcal{G}$ & Relational structure (e.g., graph/topology) conditioning the dynamics/energy \\
$E(X;\mathcal{G})$ & Total semantic energy of CDAT \\
$E^{\mathrm{ATT-vMF}}$ & Mo--vMF attention energy term \\
$E^{\mathrm{HN}}$ & Hopfield refinement energy term \\
$\lambda_v, \lambda_h$ & Balancing weights for the Mo--vMF and Hopfield energy blocks \\
\midrule
$\tau$ & Time constant setting the scale of the dynamical system \\
$\Delta t$ & Discrete time step size for Euler integration \\
$\alpha$ & Euler step size, $\alpha := \Delta t/\tau$ \\
$W$ & Topology-conditioned coupling (mutual excitation) matrix, $W\in\mathbb{R}^{D\times D}$ \\
$\omega$ & Global damping/self-inhibition scalar decay rate \\
$V(t)$ & Trajectory-wise storage functional (Lyapunov-type functional) \\
\midrule
$J_0$ & Excitation strength in the Gaussian coupling prior \\
$a$ & Receptive range (bandwidth) in the Gaussian coupling prior \\
$W^{\mathrm{train}}$ & Learnable residual in the coupling, used in $W_{ij}=J_0\exp(-d(i,j)^2/2a^2)+W^{\mathrm{train}}_{ij}$ \\
$d(i,j)$ & Geodesic (structure-induced) distance between feature indices $i,j\in\{1,\dots,D\}$ \\
\midrule
$H$ & Number of attention heads \\
$Y$ & Dimension of the internal feature space for queries and keys \\
$Q, K$ & Query/Key representations for attention (per head) \\
$P$ & Number of local minima \\
$\beta$ & Concentration parameter (inverse temperature) in the vMF distribution \\
${\mu}_b, \pi_b$ & Mean direction and mixture weight of component $b$ in Mo--vMF \\
${\xi}_\mu$ & Stored memory/prototype vector $\mu$ in the Hopfield module \\
\bottomrule
\end{tabular}
\end{table}

\section{Experimental Details}
\subsection{Dataset details}

To comprehensively evaluate the effectiveness of CDAT, we conducted experiments on two distinct tasks: graph-level classification and graph anomaly detection. We utilized a total of ten standard benchmark datasets.

For the \textbf{Graph Classification} task, we selected 7 widely used benchmark datasets from the TUDataset collection~\cite{morris2020tudataset}. These include bioinformatics datasets (PROTEINS, NCI1, NCI109, DD, ENZYMES, MUTAG, MUTAGENICITY). These datasets vary significantly in terms of graph size, number of classes, and average node density.

For the \textbf{Graph Anomaly Detection} task, we employed three large-scale fraud detection datasets: YelpChi (Yelp), Amazon, and T-Finance. For the three datasets used in the experiments, Amazon and Yelp datasets can be obtained from the DGL library, while T-Finance can be obtained from~\cite{tang2022rethinking}. These datasets involve classifying nodes as either benign or anomalous (fraudulent) based on their features and structural patterns. 

The statistics of the datasets used in our experiments are summarized in Table~\ref{tab:dataset_stats_class} and Table~\ref{tab:dataset_stats_ad}.

\begin{table*}[ht]
  \caption{The statistics and properties of the seven datasets of TUDataset (additional node attributes are indicated by '+').}
  \label{tab:dataset_stats_class}
  \begin{center}
    \begin{small}
      \begin{sc}
        \begin{tabular}{lccccc}
        \toprule
        \textbf{Dataset} & \textbf{Graphs} & \textbf{Avg. Nodes} & \textbf{Avg. Edges} & \textbf{Node Attr} & \textbf{Classes} \\
        \midrule
        MUTAG & 188 & 17.93 & 19.79 & 7 & 2 \\
        ENZYMES & 600 & 32.63 & 62.14 & 18 + 3 & 6 \\
        PROTEINS & 1113 & 39.06 & 72.82 & 0 + 4 & 2 \\
        DD & 1178 & 284.32 & 715.66 & 89 & 2 \\
        NCI1 & 4110 & 29.87 & 32.30 & 37 & 2 \\
        NCI109 & 4127 & 29.68 & 32.13 & 38 & 2 \\
        MUTAGENICITY & 4337 & 30.32 & 30.77 & 14 & 2 \\
        \bottomrule
    \end{tabular}
      \end{sc}
    \end{small}
  \end{center}
%  \vskip -0.1in
\end{table*}

\begin{table*}[ht]
  \caption{Summary of the graph anomaly detection datasets used in experiments.}
  \label{tab:dataset_stats_ad}
  \begin{center}
    \begin{small}
      \begin{sc}
        \begin{tabular}{lcccc}
        \toprule
        \textbf{Dataset} & \textbf{$|V|$} & \textbf{$|E|$} & \textbf{Anomaly(\%)} & \textbf{Features} \\
        \midrule
        Amazon & 11944 & 4398392 & 6.87\% & 25 \\
        Yelp & 45954 & 3846979 & 14.53\% & 32 \\
        T-Finance & 39357 & 21222543 & 4.58\% & 10 \\
        \bottomrule
        \end{tabular}
      \end{sc}
    \end{small}
  \end{center}
%  \vskip -0.1in
\end{table*}

\subsection{Details of CDAT training on Anomaly Detection Task}

We follow a unified training protocol across all datasets.
Specifically, we train each model for 100 epochs using the Adam optimizer with a learning rate of 0.001,
and report the average performance over 5 independent runs.
Unless otherwise specified, we use a training ratio of 0.4, with the remaining nodes reserved for validation and testing.

For the model architecture, we keep most hyperparameters fixed (Table~\ref{tab:hyperparams_subsample1}):
we use 2 layers with 2 attention heads by default,
enable LayerNorm while disabling BatchNorm, and adopt residual connections with dropout 0.1 and an FFN expansion ratio of 4.
The default hidden dimension is 64.
For the energy-iteration module, the default setting uses step size $\alpha=0.1$,
suppression coefficient 1.0, and noise standard deviation 0.02 (noise can be disabled by setting noise\_std=0).

We tune a small set of key hyperparameters on the validation set and select the best configuration within the 100-epoch budget for final test reporting.
The tunable hyperparameters include hid\_dim, num\_heads, alpha, suppression\_coef, and noise\_std.
The optimal per-dataset (and per-training-ratio) choices used in our experiments are summarized in Table~\ref{tab:tunable_hparams}.
To speed up training, we enable subgraph sampling (subsample\_flag=1) and use a fixed subsampling ratio of 0.05 per epoch
(sample\_ratio=0.05), which reduces computation while maintaining stable optimization.

\begin{table}[t]
\centering
\caption{Hyperparameter settings for node anomaly detection.}
\label{tab:hyperparams_subsample1}
\small
\begin{tabular}{lll}
\toprule
\multicolumn{3}{c}{\textbf{Training}} \\
\midrule
\textbf{Parameter} & \textbf{Value} & \textbf{Note} \\
\midrule
dataset & amazon & \na \\
train\_ratio & 0.4 & \na \\
epoch & 100 & \na \\
run & 5 & \na \\
seed & 1 & \na \\
optimizer & Adam & \na \\
lr & 0.001 & learning rate \\
sample\_ratio & 0.05 & \na \\
\midrule
\multicolumn{3}{c}{\textbf{Architecture}} \\
\midrule
\textbf{Parameter} & \textbf{Value} & \textbf{Note} \\
\midrule
hid\_dim & 64 & hidden dimension \\
order & 2 & \na \\
homo & 1 & \na \\
n\_layers & 2 & \na \\
num\_heads & 2 & \na \\
layer\_norm & True & \na \\
batch\_norm & False & \na \\
residual & True & \na \\
dropout & 0.1 & \na \\
r & 4 & the dimension of matrix Q \\
ffn & 4 & FFN expansion ratio \\
\midrule
\multicolumn{3}{c}{\textbf{Energy Iteration}} \\
\midrule
\textbf{Parameter} & \textbf{Value} & \textbf{Note} \\
\midrule
alpha & 0.1 & energy iteration step size \\
suppression\_coef & 1.0 & suppression coefficient \\
noise\_std & 0.02 & base noise std; 0 disables noise \\
\midrule
\multicolumn{3}{c}{\textbf{Other}} \\
\midrule
\textbf{Parameter} & \textbf{Value} & \textbf{Note} \\
\midrule
num\_class & 2 & \na \\
subsample\_flag & 1 & \na \\
ablation\_mode & False & \na \\
\bottomrule
\end{tabular}
\end{table}

\begin{table}[t]
\centering
\caption{Tunable hyperparameters used in each setting.}
\label{tab:tunable_hparams}
\small
\begin{tabular}{l c c c c c c}
\toprule
\textbf{Dataset} &
\textbf{Train ratio} &
\textbf{hid\_dim} &
\textbf{num\_heads} &
\textbf{alpha} &
\textbf{suppression\_coef} &
\textbf{noise\_std} \\
\midrule
amazon   & 0.4  & 256 & 2 & 0.01 & 0.1 & 0 \\
amazon   & 0.01 & 256 & 2 & 0.01 & 0.5 & 0 \\
amazon   & 0.7  & 256 & 4 & 0.1  & 0.1 & 0 \\
tfinance & 0.7  & 128 & 4 & 0.01 & 0.1 & 0 \\
tfinance & 0.4  & 128 & 4 & 0.01 & 0.5 & 0 \\
tfinance & 0.01 & 128 & 4 & 0.01 & 1.0 & 0.02 \\
yelp     & 0.4  & 128 & 4 & 0.01 & 0.1 & 0 \\
yelp     & 0.7  & 256 & 2 & 0.05 & 0.1 & 0.02 \\
yelp     & 0.01 & 256 & 4 & 0.01 & 1.0 & 0 \\
\bottomrule
\end{tabular}
\end{table}

\subsection{Details of CDAT training on Graph Classification Task}
We train CDAT for 100 epochs with the Adam optimizer ($b_1{=}0.9$, $b_2{=}0.99$) using a peak learning rate of $10^{-3}$,
a warmup period of 50 epochs, and a cosine-decay schedule to an initial/ending learning rate of $5\times 10^{-6}$.
We use a batch size of 64, weight decay 0.05, and do not apply gradient clipping.
All experiments are conducted on a single GPU device.
The full set of training and architecture hyperparameters is summarized in Table~\ref{tab:et_tudataset_dd}.

For the model architecture, we set the token/embedding dimension to 128 with 12 attention heads of head dimension 64.
The softmax inverse temperature is initialized as $\beta=1/\sqrt{64}$ and is learned during training (train\_betas=Yes).
CDAT uses an energy-iteration step size $\alpha=0.1$ with Gaussian noise of standard deviation $\sigma_{\varepsilon}=0.02$.
For graph spectral processing, we retain the top $k=15$ eigenvalues.
The network depth is 1 with 4 blocks, using kernel size [3, 3] and dilation size [1, 1].
The Hopfield refinement module adopts a hidden dimension of 512 (i.e., $4\times 128$), without bias in the Hopfield and attention modules,
while LayerNorm keeps bias enabled.
We fix the number of tokens to 500, use ReLU as the channel activation (chn\_atype=relu),
and enable correlation computation (compute\_corr=True).

\begin{table}[ht]
\centering
\caption{Hyperparameter and architecture choices for CDAT during TUDataset experiments.}
\label{tab:et_tudataset_dd}
\small
\begin{tabular}{llll}
\toprule
\multicolumn{2}{c}{\textbf{Training}} & \multicolumn{2}{c}{\textbf{Architecture}} \\
\midrule
\textbf{Parameter} & \textbf{Value} & \textbf{Parameter} & \textbf{Value} \\
\midrule
batch\_size & 64 & token\_dim & 128 \\
epochs & 100 & num\_heads & 12 \\
peak lr & $10^{-3}$ & head\_dim & 64 \\
warmup\_epochs & 50 & $\beta$ & $1/\sqrt{64}$ \\
initial and ending lr & $5\times 10^{-6}$ & train\_betas & Yes \\
$b_1, b_2$ (Adam) & 0.9, 0.99 & step size $\alpha$ & 0.1 \\
weight\_decay & 0.05 & $k$ eigenvalues & 15 \\
grad\_clipping & None & noise $\sigma_{noise}$ & 0.02 \\
num.\ of gpu devices & 1 & depth & 1 \\
r & 4 & block & 4 \\
\na & \na & kernel\_size & [3, 3] \\
\na & \na & dilation\_size & [1, 1] \\
\na & \na & hidden\_dim (HN) & 512 \\
\na & \na & bias in HN & None \\
\na & \na & bias in ATT-vMF & None \\
\na & \na & bias in LNORM & Yes \\
\na & \na & num\_tokens & 500 \\
\na & \na & chn\_atype & relu \\
\na & \na & compute\_corr & True \\
\na & \na & avg.\ total \#params & 1{,}066{,}855 \\
\bottomrule
\end{tabular}
\end{table}

\section{Ablation Study}

Table~\ref{tab:ablation_graph_cls} reports the ablation results on \emph{graph-level classification}.
We evaluate the two proposed dynamical components by disabling them individually.
Removing \emph{mutual excitation} leads to a clear performance drop on all datasets, with the most pronounced degradation on DD/NCI1/NCI109, indicating that structure-aware excitatory coupling is crucial for effective graph representation aggregation.
In contrast, removing \emph{self-inhibition} yields a smaller but still consistent decrease, suggesting that inhibitory control stabilizes the iterative dynamics and improves robustness across benchmarks.
Combining both components (\textbf{CDAT} Full) achieves the best results, confirming that mutual excitation provides the primary gain while self-inhibition offers complementary stabilization and refinement.

% Appendix table: ablation on graph classification
\begin{table}[t]
  \caption{Ablation study on graph classification benchmarks. Accuracy (\%) is reported.}
  \label{tab:ablation_graph_cls}
  \centering
  \begin{small}
   \begin{tabular}{lcccc}
    \toprule
    \textbf{Setting (ablation)} & \textbf{DD} & \textbf{NCI1} & \textbf{NCI109} & \textbf{MUTAG} \\
    \midrule
    Mutual Excitation & 83.09 & 82.84 & 84.89 & 96.27 \\
    Self-Inhibition   & 98.09 & 92.74 & 92.07 & 98.46 \\
    \midrule
    \textbf{CDAT(Full)} & \textbf{98.50} & \textbf{93.50} & \textbf{94.30} & \textbf{98.80} \\
    \bottomrule
  \end{tabular}
  \end{small}
\end{table}

In addition to the graph-level ablation in Table~\ref{tab:ablation_graph_cls}, we further conduct a more exhaustive study on anomaly detection by comparing \textbf{CDAT} variants against the strong baseline ET (Table~\ref{tab:ablation_anomaly_mf1_auc}).
Specifically, we isolate the two key dynamical terms by keeping only the suppression/inhibition component (\textbf{SUPP}) or only the excitation-driven interaction component (\textbf{Wx}), and report MF1/AUC across different anomaly ratios in Table~\ref{tab:ablation_anomaly_mf1_auc}.
Overall, neither isolated term alone consistently matches the full dynamics across datasets and ratios: \textbf{Wx} typically retains stronger discriminative ability, while \textbf{SUPP} can be less stable under harder settings (e.g., higher ratios), highlighting the necessity of combining excitation and inhibition (Table~\ref{tab:ablation_anomaly_mf1_auc}).
Together with the ET comparison, these results provide a comprehensive validation that the gains of \textbf{CDAT} come from the proposed dynamical design rather than ad-hoc architectural changes (Table~\ref{tab:ablation_anomaly_mf1_auc}).

\begin{table*}[t]
  \caption{Comprehensive ablation on anomaly detection across different anomaly ratios. We report MF1 and AUC (\%, mean $\pm$ std). ET is the baseline; \textbf{SUPP} and \textbf{Wx} are term-wise variants; \textbf{CDAT (Full)} corresponds to the complete model.}
  \label{tab:ablation_anomaly_mf1_auc}
  \centering
  \begin{small}
  \begin{tabular}{l c c c c c}
    \toprule
    \textbf{Dataset} & \textbf{Ratio} &
    \textbf{ET} & \textbf{SUPP} & \textbf{Wx} & \textbf{CDAT (Full)} \\
    \midrule
    \multicolumn{6}{c}{\textbf{MF1 (\%)}} \\
    \midrule
    amazon   & 0.01 & $89.42\pm5.80$ & $90.04\pm1.60$ & $89.42\pm5.80$ & \textbf{$91.2\pm0.9$} \\
    amazon   & 0.40 & $90.39\pm1.02$ & $91.39\pm1.59$ & $89.90\pm1.67$ & \textbf{$92.1\pm0.6$} \\
    amazon   & 0.70 & \textbf{$92.05\pm3.58$} & $85.82\pm10.81$ & $90.99\pm3.83$ & $91.9\pm0.1$ \\
    tfinance & 0.01 & $86.09\pm1.42$ & $59.53\pm6.49$ & $85.66\pm1.21$ & \textbf{$87.4\pm1.1$} \\
    tfinance & 0.40 & \textbf{$91.49\pm2.16$} & $86.42\pm2.50$ & \textbf{$91.49\pm2.16$} & $90.6\pm0.6$ \\
    yelp     & 0.01 & $62.74\pm1.93$ & $61.99\pm0.73$ & $61.94\pm0.88$ & \textbf{$63.0\pm1.4$} \\
    yelp     & 0.40 & $69.52\pm0.22$ & $68.74\pm2.05$ & $69.91\pm1.34$ & \textbf{$70.5\pm0.1$} \\
    yelp     & 0.70 & $70.30\pm3.69$ & $64.57\pm1.32$ & $69.21\pm0.96$ & \textbf{$71.5\pm0.2$} \\
    \midrule
    \multicolumn{6}{c}{\textbf{AUC (\%)}} \\
    \midrule
    amazon   & 0.01 & $94.51\pm1.19$ & \textbf{$95.43\pm1.78$} & $94.51\pm1.14$ & $94.5\pm1.1$ \\
    amazon   & 0.40 & $95.72\pm2.02$ & $95.10\pm1.23$ & $96.04\pm2.29$ & \textbf{$97.2\pm0.2$} \\
    amazon   & 0.70 & \textbf{$97.07\pm4.15$} & $92.90\pm5.03$ & \textbf{$97.07\pm4.15$} & $93.0\pm5.0$ \\
    tfinance & 0.01 & $92.25\pm1.11$ & $79.04\pm3.08$ & $92.68\pm1.72$ & \textbf{$95.9\pm0.4$} \\
    tfinance & 0.40 & $96.72\pm1.70$ & $93.65\pm1.02$ & $97.23\pm1.98$ & \textbf{$97.8\pm1.1$} \\
    yelp     & 0.01 & $72.86\pm1.26$ & $73.05\pm1.43$ & $73.30\pm1.59$ & \textbf{$74.3\pm0.3$} \\
    yelp     & 0.40 & $83.57\pm2.75$ & $81.70\pm2.08$ & $83.45\pm2.69$ & \textbf{$84.4\pm0.2$} \\
    yelp     & 0.70 & $81.65\pm0.96$ & $78.24\pm4.52$ & $82.22\pm0.32$ & \textbf{$91.9\pm0.3$} \\
    \bottomrule
  \end{tabular}
  \end{small}
\end{table*}

\section{Parameter Comparison}

Tables~\ref{tab:param_compare_cdat_et} and~\ref{tab:param_compare_graphcls} compare the parameter counts of our full model \textbf{CDAT} and the \textbf{ET} ablation.
On the anomaly detection setting (Tables~\ref{tab:param_compare_cdat_et}), CDAT and ET have identical parameters per Transformer block, while CDAT has a slightly larger total parameter count.
This is expected: our improvement is not introduced by modifying the internal architecture of each block, but by adding extra dynamical control terms (state-dependent modulation variables) outside the standard block parameterization.
Therefore, the additional parameters appear only in the global model budget, leaving the per-block parameter count unchanged.
On graph classification (Tables~\ref{tab:param_compare_graphcls}), we observe a similarly small gap between CDAT and ET, indicating that the proposed dynamics incurs only a negligible parameter overhead overall.

\begin{table}[ht]
\centering
\caption{Comparison between the number of parameters in our full model \textbf{CDAT} and the \textbf{ET} on anomaly detection task.}
\label{tab:param_compare_cdat_et}
\begin{tabular}{lcc}
\toprule
\textbf{Model} & \textbf{NParams} & \textbf{NParams (per block)} \\
\midrule
CDAT & 4.29M \hspace{4pt} $\downarrow$0.00\% & 0.59M \hspace{4pt} $\downarrow$0.00\% \\
ET   & 4.25M \hspace{4pt} $\downarrow$1.11\% & 0.59M \hspace{4pt} $\downarrow$0.00\% \\
\bottomrule
\end{tabular}
\end{table}

\begin{table}[t]
\centering
\caption{Parameter comparison between \textbf{CDAT} and \textbf{ET} on graph classification task.}
\label{tab:param_compare_graphcls}
\begin{tabular}{lc}
\toprule
\textbf{Model} & \textbf{NParams} \\
\midrule
CDAT & 1.060M \hspace{4pt} $\downarrow$0.00\% \\
ET   & 1.058M \hspace{4pt} $\downarrow$0.19\% \\
\bottomrule
\end{tabular}
\end{table}

\section{Additional Proofs}

\subsection{Proof of Theorem~\ref{prop:resgcn_equivalence}}
\label{app:gcn-proof}
\begin{proof}
Consider the continuous-time dynamics of a single token state $\mathbf{x}(t) \in \mathbb{R}^D$. In the context of Continuous Attractor Neural Networks (CANNs), the indices $i \in \{1, \cdots, D\}$ do not represent independent semantic features, but rather discrete coordinates on a latent topological manifold (\emph{e.g.}, a grid or a graph).

The mutual excitation component of the dynamics is given by $\tau \dot{\mathbf{x}} = W\mathbf{x}$. Applying a first-order Euler discretization with step size $\Delta t$ and $\alpha:=\Delta t/\tau$, we obtain the update rule for the $i$-th neuron:
\begin{equation}
    x_i^{(t+1)} = x_i^{(t)} + \alpha \sum_{j=1}^{D} W_{ij} x_j^{(t)}.
\end{equation}

In our formulation, the coupling matrix $W \in \mathbb{R}^{D \times D}$ is parameterized as a kernelized prior plus a learnable residual:
\begin{equation}
    W_{ij} = J_0 \exp\left(-\frac{d(i,j)^2}{2a^2}\right) + W_{ij}^{train},
\end{equation}
where $d(i,j)$ is the geodesic distance between neuron $i$ and neuron $j$ on the feature manifold.

Let $\mathcal{K}_{ij} = J_0 \exp(-d(i,j)^2 / 2a^2)$. The term $\sum_{j} \mathcal{K}_{ij} x_j$ represents a discrete approximation of the continuous convolution of the signal $x$ with a Gaussian kernel:
\begin{equation}
    (x * \mathcal{K})(i) \approx \sum_{j=1}^D \mathcal{K}_{ij} x_j.
\end{equation}
In Graph Signal Processing (GSP), this operation is precisely a \textbf{spatial graph convolution} (or heat diffusion) on the graph defined by the neurons, with $\mathcal{K}$ acting as the diffusion filter.

Furthermore, we can define an effective adjacency matrix $\tilde{A}$ for this latent graph where $\tilde{A}_{ij} = \alpha W_{ij}$. The update rule then becomes:
\begin{equation}
    \mathbf{x}^{(t+1)} = \mathbf{x}^{(t)} + \tilde{A} \mathbf{x}^{(t)} = (I + \tilde{A}) \mathbf{x}^{(t)}.
\end{equation}
This form is algebraically equivalent to a layer of a \textbf{Residual Graph Convolutional Network (ResGCN)} operating on the \textit{feature graph} (where nodes are feature dimensions), rather than the input data graph.

Thus, the mutual excitation term strictly enforces a topological prior: it diffuses activity among neighboring neurons in the latent space, smoothing the representation locally according to the manifold structure defined by $d(i,j)$, while the learnable term $W^{train}$ allows for non-local "shortcut" connections similar to long-range dependencies in modern GCNs. This proves that the update step is a residual convolution on the feature manifold.
\end{proof}

\subsection{Proof of Theorem~\ref{thm:dissipation_certificate}}
\label{app:dissipation-proof}
\begin{proof}
Let $\mathbf{x}(\cdot)$ be any $C^1$ solution of the dynamics and define $\mathbf{g}(t):=\mathbf{g}(\mathbf{x}(t))$.
Assume $E,\mathbf{g}\in C^1$ and $W=W^\top$.
Fix $t_0$ and define the trajectory-wise storage functional
\begin{equation}
\label{eq:appendix_Vt}
V(t)
:= E(\mathbf{g}(t))
+\int_{t_0}^{t}
\Bigl((1+\omega)\mathbf{x}(s)-W \mathbf{x}(s)\Bigr)^\top \dot{ \mathbf{g}}(s)\,ds,
\end{equation}
where $\dot{ \mathbf{g}}(s)=\frac{d}{ds}\mathbf{g}(\mathbf{x}(s))$.
Since $\mathbf{x}(\cdot)$ is $C^1$ and $\mathbf{g}\in C^1$, the integrand is continuous, hence the integral is well-defined.

By the chain rule,
\begin{equation}
\label{eq:chain_rule_E}
\frac{d}{dt}E(\mathbf{g}(t))=\nabla_\mathbf{g} E(\mathbf{g}(t))^\top  \dot{ \mathbf{g}}(t).
\end{equation}
Applying Leibniz' rule to the time integral in~\eqref{eq:appendix_Vt} yields
\begin{equation}
\label{eq:leibniz}
\frac{d}{dt}\int_{t_0}^{t}
\Bigl((1+\omega)\mathbf{x}(s)-W \mathbf{x}(s)\Bigr)^\top \dot{ \mathbf{g}}(s)\,ds
=
\Bigl((1+\omega)\mathbf{x}(t)-W \mathbf{x}(t)\Bigr)^\top \dot{ \mathbf{g}}(t).
\end{equation}
Combining~\eqref{eq:chain_rule_E} and~\eqref{eq:leibniz}, we obtain
\begin{equation}
\label{eq:Vdot_pre}
\dot V(t)
=
\Bigl(\nabla_\mathbf{g} E(\mathbf{g}(t))+(1+\omega)\mathbf{x}(t)-W \mathbf{x}(t)\Bigr)^\top \dot{ \mathbf{g}}(t).
\end{equation}

Rearranging the dynamics gives
\begin{equation}
\label{eq:dyn_rearrange}
\nabla_\mathbf{g} E(\mathbf{g}(t))+(1+\omega)\mathbf{x}(t)-W \mathbf{x}(t)
=
-\tau\,\dot{\mathbf{x}}(t).
\end{equation}
Substituting~\eqref{eq:dyn_rearrange} into~\eqref{eq:Vdot_pre} yields
\begin{equation}
\label{eq:Vdot_xg}
\dot V(t)=-\tau\,\dot{\mathbf{x}}(t)^\top \dot{\mathbf{g}}(t).
\end{equation}
Let  $M( \mathbf{x}):=\nabla_{ \mathbf{x}} \mathbf{g}( \mathbf{x})$. Therefore,
\begin{equation}
\label{eq:Vdot_final}
\dot V(t)
=
-\tau\,\dot{\mathbf{x}}(t)^\top M(\mathbf{x})\,\dot{\mathbf{x}}(t).
\end{equation}
If $M(\mathbf{x}(t))\succeq 0$ (e.g., by Lemma~\ref{lem:LN_Jacobian_PSD} for LayerNorm with $\gamma>0$),
then $\dot V(t)\le 0$, hence $V(t)$ is non-increasing.
\end{proof}

\subsection{Proof of Theorem~\ref{prop:movmf}}
\label{app:movmf-proof}

\begin{proof}
By the Boltzmann principle, the probability of a configuration indexed by
$l$ is proportional to the exponential of minus its energy,
\begin{equation}
\label{eq:boltzmann}
p_l \propto \exp(-\beta E_l).
\end{equation}

Now consider a mixture of von Mises--Fisher distributions on $S^{Y-1}$
with density
\begin{equation}
\label{eq:movmf}
p(x)
=
\sum_{B=1}^{N}
\pi_B\, C_Y(\kappa)\,
\exp\bigl(\kappa\,\mu_B^\top x\bigr),
\end{equation}
where all components share the same concentration parameter
$\kappa>0$, the mean directions satisfy $\|\mu_B\|=1$, and
$C_Y(\kappa)$ denotes the normalization constant.
Moreover, the mask $B\neq C$ in~\eqref{eq:att_energy} corresponds to a
query-dependent mixture that excludes the $B=C$ component, \emph{i.e.},
\[
\pi_{B|C}=\begin{cases}
\frac{1}{N-1}, & B\neq C,\\
0, & B=C.
\end{cases}
\]
To match the attention log-sum-exp over keys for each query, we identify
\[
\kappa = \beta,
\qquad
x = K_C,
\qquad
\mu_B = Q_B.
\]
Under this identification, the negative log-marginal likelihood of
observing $K_C$ becomes
\begin{align}
-\log p(K_C)
&=
-\log\!\left(
\sum_{B\neq C}
\frac{1}{N-1} C_Y(\beta)\exp\bigl(\beta\, Q_B^\top K_C\bigr)
\right) \notag\\
&=
-\log\!\left(
\sum_{B\neq C}\exp\bigl(\beta\, Q_B^\top K_C\bigr)
\right)
+\mathrm{const},
\label{eq:nll_single}
\end{align}
where the constant absorbs $\log C_Y(\beta)$ and $\log N$.

Comparing~\eqref{eq:nll_single} with the per-token contribution in
~\eqref{eq:att_energy}, we observe that the self-attention energy
associated with each key $K_C$ is proportional to the negative
log-likelihood under the Mo--vMF model. According to the Boltzmann
principle, we know
\begin{equation}
E^{K_C} \propto -\log p(K_C).
\end{equation}
Aggregating over all keys yields the global energy
\begin{equation}
\label{eq:global_energy}
E^{\mathrm{global}}
=
\sum_{C=1}^{N} E^{K_C}
\propto
-\frac{1}{\beta}
\sum_{C=1}^{N}
\log\!\left(
\sum_{B\neq C}
\exp\bigl(\beta\, Q_B^\top K_C\bigr)
\right),
\end{equation}
which matches the self-attention energy
\eqref{eq:attn-movmf} up to global constants.
\end{proof}

\section{Algorithm}

\begin{algorithm}[tb]
\caption{Training and inference pseudocode of \textsc{CDAT}}
\label{alg:cdat_al}
\begin{algorithmic}[1]

\STATE \textbf{HyperParameters}
\STATE rollout steps $T$; step size $\Delta t$; time constant $\tau$ (thus $\alpha=\Delta t/\tau$);
epochs $N_{\mathrm{epoch}}$; batch size $B$; inverse temperature $\beta$; learning rate $\eta$

\STATE \textbf{Parameters}
\STATE feature coupling $W\in\mathbb{R}^{D\times D}$ 
\STATE attention kernels $W^Q,W^K\in\mathbb{R}^{Y\times H\times D}$; Hopfield memory $\xi\in\mathbb{R}^{P\times D}$
\STATE decay rate $\omega \in \mathbb{R}_{>0}$; LayerNorm params  $(\gamma_{\mathrm{norm}},\delta_{\mathrm{norm}})$; balancing weights  $\lambda_{\mathrm{v}},\lambda_{\mathrm{h}}$ 
\STATE decoder parameters $\Psi$
\STATE dynamics parameters $\Theta$\\
\STATE[]{\noindent\rule{\linewidth}{0.4pt}}
   \STATE \textbf{Infer} \hfill // state dynamics rollout\\
   \STATE \textbf{Inputs:} corrupted/initial tokens $X_0\in\mathbb{R}^{N\times D}$
   \STATE Initialize $X \leftarrow X_0$, $\alpha \leftarrow \Delta t/\tau$
   \FOR{$t=0$ \textbf{to} $T-1$}
     \STATE $G \leftarrow \mathrm{LayerNorm}(X;\gamma_{\mathrm{norm}},\delta_{\mathrm{norm}})$
     \STATE $Q \leftarrow G(W^Q)^\top$, $K \leftarrow G(W^K)^\top$
     \STATE $E^{\mathrm{HN}} \leftarrow - \sum_{b=1}^{N}\sum_{\mu=1}^{P} \mathrm{ReLU}\!\left(\sum_{j=1}^{D} \xi_{\mu j}\, G_{b j}\right)$
     \STATE $E^{\mathrm{ATT-vMF}} \leftarrow -\frac{1}{\beta} \sum_{c=1}^{N}
      \log\!\left(\sum_{b\neq c} \exp\!\bigl(\beta\, Q_{b}^{\top}K_{c}\bigr)\right)$
     \STATE $E \leftarrow \lambda_{\mathrm{v}} E^{\mathrm{ATT-vMF}} + \lambda_{\mathrm{h}} E^{\mathrm{HN}}$
     \STATE $U \leftarrow \nabla_{G} E$
     \STATE $F_{\mathrm{exc}} \leftarrow WX$
     \STATE $F_{\mathrm{dec}} \leftarrow -(1+\omega)X$
     \STATE $X \leftarrow X + \alpha\bigl(F_{\mathrm{exc}} + F_{\mathrm{dec}} - U\bigr)$
   \ENDFOR
   \STATE \textbf{Output:} refined tokens $X_T \leftarrow X$\\
\STATE[]{\noindent\rule{\linewidth}{0.4pt}}
\STATE \textbf{Train}\\
   \STATE \textbf{Inputs:} training set $\mathcal{S}=\{(X_0^{(n)}, y^{(n)})\}$ \hfill // $X$ is the clean target tokens
   \FOR{$\text{epoch}=1$ \textbf{to} $N_{\mathrm{epoch}}$}
     \FOR{\textbf{each} mini-batch $\mathcal{B}\subset\mathcal{S}$}
       \STATE $\mathcal{L}_{\mathcal{B}} \leftarrow 0$
       \FOR{\textbf{each} $(X_0,y)\in \mathcal{B}$}
         \STATE $X_T \leftarrow \textsc{Infer}(X_0)$
         \STATE $\hat{y} \leftarrow \mathrm{Dec}_{\Psi}(X_T)$ \hfill // decode refined tokens
         \STATE $\mathcal{L}_{\mathcal{B}} \leftarrow \mathcal{L}_{\mathcal{B}} + \mathrm{MSE}(\hat{y}, y)$
       \ENDFOR
       \STATE $\mathcal{L}_{\mathcal{B}} \leftarrow \mathcal{L}_{\mathcal{B}}/|\mathcal{B}|$
       \STATE $\Theta \leftarrow \Theta - \eta\,\nabla_{\Theta}\mathcal{L}_{\mathcal{B}}$, $\Psi \leftarrow \Psi - \eta\,\nabla_{\Psi}\mathcal{L}_{\mathcal{B}}$
     \ENDFOR
   \ENDFOR
   \STATE \textbf{Return} trained parameters $(\Theta,\Psi)$

\end{algorithmic}
\end{algorithm}

%%%%%%%%%%%%%%%%%%%%%%%%%%%%%%%%%%%%%%%%%%%%%%%%%%%%%%%%%%%%%%%%%%%%%%%%%%%%%%%
%%%%%%%%%%%%%%%%%%%%%%%%%%%%%%%%%%%%%%%%%%%%%%%%%%%%%%%%%%%%%%%%%%%%%%%%%%%%%%%

\end{document}